\documentclass{article}

\usepackage[main, final]{neurips_2026}

\usepackage[utf8]{inputenc}
\usepackage[T1]{fontenc}
\usepackage{hyperref}
\usepackage{url}
\usepackage{booktabs}
\usepackage{amsfonts}
\usepackage{amsmath}
\usepackage{amssymb}
\usepackage{amsthm}
\usepackage{nicefrac}
\usepackage{microtype}
\usepackage{xcolor}
\usepackage{graphicx}
\usepackage{multirow}
\usepackage{bm}
\usepackage{algorithm}
\usepackage{algpseudocode}
\usepackage{enumitem}
\usepackage{array}
\usepackage{tikz}
\usetikzlibrary{calc}
\usepackage{float}

\theoremstyle{plain}
\newtheorem{theorem}{Theorem}

\newtheorem{proposition}[theorem]{Proposition}
\newtheorem{corollary}[theorem]{Corollary}

\theoremstyle{definition}

\theoremstyle{remark}

\newcolumntype{C}[1]{>{\centering\arraybackslash}p{#1}}
\newcolumntype{R}[1]{>{\raggedleft\arraybackslash}p{#1}}

\newcommand{\dout}{d_{\text{out}}}
\newcommand{\din}{d_{\text{in}}}
\newcommand{\LoRASpace}{\textsc{LoRASpace}}

\title{\LoRASpace{}: A Pool-Wide Shared Substrate for\\
Static and Dynamic Multi-LoRA Composition}

\author{%
  \parbox{0.95\textwidth}{\centering
    Thamirawaran~Sathiyalogeswaran\textsuperscript{1},\;
    Navindu~De~Silva\textsuperscript{1},\;
    Chamikara~Siriwardane\textsuperscript{1},\;
    W.~K.~Alwis\textsuperscript{1},\;
    Jason~Mars\textsuperscript{2},\;
    Udaya~Sampath~K.~Perera~Miriya~Thanthrige\textsuperscript{1},\;
    Ranga~Rodrigo\textsuperscript{1}
    \\[8pt]
    \mdseries
    \textsuperscript{1}University of Moratuwa, Sri Lanka \qquad
    \textsuperscript{2}University of Michigan, USA
  }%
}

\begin{document}

\maketitle

\begin{abstract}
Low-rank adaptation (LoRA) attaches small low-rank matrices to a frozen backbone, and public repositories like Hugging Face host many such LoRAs per backbone, each for a different concept, character, style, or task. Users routinely want to combine several LoRAs in a single generation, but two existing strategies involve a trade-off with each other. First, combining the LoRA weights into a single update before inference (e.g., static merge) needs only one forward pass, but quality drops, because LoRAs trained for different concepts interfere when their weights are summed into a single update, distorting each LoRA's contribution. Second, running each LoRA in a separate forward pass and summing the outputs (e.g., composite) preserves output quality, but each LoRA adds a full backbone pass, so inference time scales linearly with the active LoRA count. To recover composite-level quality at merge cost, we study a property of LoRA pools overlooked by prior multi-LoRA work. In real LoRA pools, the stacked low-rank factors need far fewer directions than their summed ranks. Using singular value decomposition, we find the minimal subspace that preserves 99\% of the factor energy. On three diffusion pools, only 46\% to 57\% of the sum of individual LoRA ranks is enough to represent the pool, and the subspace grows sub-linearly with pool size. A 48-LoRA sub-pool of LoRA-Hub on FLAN-T5 NLP tasks shows comparable overlap (about 62\%), while the diversity-curated LoRA-Retriever benchmark sits near the no-overlap baseline at 96\%, indicating that across these five pools pool composition rather than modality predicts where compression is large. We use this overlap to construct \LoRASpace{}, a training-free pool-wide method where the subspace matrices are computed once per pool and reused across two inference modes: a static mode at merge inference cost, and a dynamic mode for input-dependent routing at substantially lower cost than composite. On the diffusion pools, static \LoRASpace{} matches composite quality at merge cost, and dynamic \LoRASpace{} achieves more than $6\times$ speedup over composite at six active LoRAs on a 17-LoRA pool, mostly by running one backbone pass instead of six. In our ablation, static variants at the same rank match \LoRASpace{}, so the quality gain does not depend on the pool-wide basis. On the near-orthogonal LoRA-Retriever benchmark, the framework predicts limited compression benefit, with gains coming from external retrieval coupling. Code is available at \url{https://github.com/Thamirawaran/loraspace}.
\end{abstract}

\section{Introduction}
\label{sec:intro}

Low-Rank Adaptation (LoRA)~\citep{hu2022lora} has made the specialization of existing language and diffusion image generation models feasible by training small low-rank \emph{adapters} that attach to a frozen backbone on demand, rather than updating the full set of model weights. The simplicity of this construction has produced a second-order effect, where open repositories such as CivitAI~\citep{civitai2023} and Hugging Face~\citep{huggingface2016} now host LoRA pools containing many trained LoRA adapters per backbone, each capturing a distinct visual concept, character, style, or downstream language task. As these libraries grow into the thousands, attention has shifted from training a single LoRA adapter (\emph{LoRAs} for short) to using several at once. A user generating a portrait may want one LoRA for the subject's identity, a second for clothing, and a third for the background, all active during a single image. 

Two ambient strategies have emerged for this multi-LoRA problem. The first, static merge~\citep{ilharco2023taskarith,yadav2023ties,yu2024dare}, folds the chosen LoRAs into a single weight update (\emph{delta} for short) before inference, so only one forward pass is required. Static merge is cheap to deploy but tends to lose \emph{individual} LoRA identity, since unrelated low-rank directions are summed in the ambient parameter space. The second strategy, often called composite~\citep{zhong2024composlora,zou2025cmlora}, runs one forward pass per active LoRA and aggregates the outputs. Composite preserves quality more faithfully but pays linearly in \emph{compute} and \emph{memory} with the number of active LoRAs. On modern diffusion backbones, this can mean several minutes per image at high LoRA counts, with proportional cost growth on language models. The two strategies sit at opposite corners of a quality-cost trade-off because both treat each LoRA as an independent low-rank addition in ambient space.

\paragraph{LoRASpace.} We observe that both the merge and composite views ignore a structural property of real LoRA pools. When the per-LoRA factors of a pool are stacked and decomposed jointly, the resulting subspace is much smaller than the sum of individual ranks. On language LoRAs trained on related tasks this reflects alignment between adapters, while on diffusion pools it comes largely from layer-width saturation and each adapter's own spectral tail (\S\ref{sec:experiments:overlap}). We capture this with the rank-fill ratio $\rho = D / \sum_i r_i$, where $D$ is the rank-truncated subspace dimension obtained by an energy-truncated SVD of the stacked factors, and $\sum_i r_i$ is the no-overlap upper bound. Based on this overlap, we introduce \LoRASpace{}, a training-free composition operator whose goal is to recover composite-level quality at near-merge inference cost while supporting input-dependent routing, built from a rank-truncated pool-wide substrate $(U, V)$ computed once at pool-creation time and reused for any active subset of LoRAs chosen at inference. The substrate supports two deployment modes from one cache. A \emph{static mode} folds time-invariant routing into a single combined delta at ambient-merge inference cost. A \emph{dynamic mode} accepts context-dependent routing through a router function $\{w_i\}$ that maps the current generation context (token, step, pixel, or prompt) to per-LoRA weights, at substantially lower cost than composite. On diffusion pools, \LoRASpace{} matches or exceeds composite quality at merge cost ($+0.74$ CLIP-T on Pool A) and reaches $6.1\times$ speedup over composite at high LoRA counts, mostly by running one backbone pass instead of one per LoRA. On the near-orthogonal LoRA-Retriever LLM benchmark, the framework predicts no compression speedup, and accuracy gains over ambient mixture come from pairing the substrate with external retrieval ($+0.96$ Macro Task).

\paragraph{Contributions.}
\begin{itemize}[leftmargin=*,topsep=2pt,itemsep=1pt]
    \item \textbf{Cross-modality rank-fill measurement.} Among the first measurements, to our knowledge, of the pool-wide rank-fill ratio $\rho = D / \sum_i r_i$ at energy threshold $\tau = 0.99$ on real LoRA pools: $\rho \le 0.57$ on diffusion pools with sub-linear scaling and $\rho \approx 0.62$ on a related-task LLM pool (LoRA-Hub at $K{=}48$), versus $\rho \approx 0.96$ on the diversity-curated LoRA-Retriever benchmark, with pool composition rather than modality predicting where compression is large.
    \item \textbf{Dual-mode (static + dynamic) pool-wide substrate.} \LoRASpace{} is a rank-truncated pool-wide operator supporting two deployment modes from one precompute. Static-mode folds time-invariant routing into a single combined delta at ambient-merge cost. Dynamic-mode evaluates context-conditioned routing (per-token, per-step, per-pixel, per-prompt) through the same shared bases. The dynamic-mode interface is router-pluggable, and we evaluate four routers (\S\ref{sec:experiments:routing}).
    \item \textbf{Mechanism ablation.} The quality lift over ambient merge is rank-$D$ truncation acting as regularisation. The pool-wide basis is \emph{approximately combo-optimal at that rank}, so one SVD amortises across all combinations and supports dynamic routing without per-combo SVDs (\S\ref{sec:experiments:quality}).
\end{itemize}

\section{Related Work}
\label{sec:related}

\providecommand{\xmark}{\ensuremath{\text{\sffamily X}}}
\begin{table}[t]
\centering
\scriptsize
\setlength{\tabcolsep}{1pt}
\caption{Where \LoRASpace{} sits among training-free multi-LoRA methods. Columns: pool-wide basis $(U, V)$ (defined in \S\ref{sec:method:operator}), active LoRAs per step, deployment mode (static vs.\ dynamic), and per-image cost normalised to ambient merge ($1\times$). To our knowledge, no other reviewed family combines a pool-wide basis with simultaneous $N{\ge}2$ composition and both deployment modes from one precompute (Eq.~\eqref{eq:loraspace_operator}).}
\label{tab:design_space}
{%
\begin{tabular}{@{}p{8cm}C{0.9cm}C{0.9cm}C{1cm}R{2.8cm}@{}}
\toprule
Method family (representative work) & Pool-wide basis & Active per step & Mode & Per-image cost \\
\midrule
Static merge (Task Arith.~\citep{ilharco2023taskarith}, TIES~\citep{yadav2023ties}, DARE~\citep{yu2024dare}) & \xmark & $N$ & static & $1\times$ \\
Shared-basis merge (Knots~\citep{stoica2025knots}, Core-Space~\citep{corespace2025}) & per merge & $N$ & static & $1\times$ \\
Composite \& step/mask routing & \xmark & $N$ & dynamic & $\sim N\times$ \\
\quad (ComposLoRA~\citep{zhong2024composlora}, CMLoRA~\citep{zou2025cmlora}, FreeFuse~\citep{freefuse2025}, MultLFG~\citep{multlfg2025}) & & & & \\
Per-input router-bound composition & \xmark & $1$ to $N$ & dynamic & $\sim N\times$ \\
\quad (SpectR~\citep{fleshman2025spectr}, K-LoRA~\citep{klora2025}, FreeLoRA~\citep{freelora2025}, LoRAHub~\citep{huang2024lorahub}, LoRA-Retriever~\citep{zhao2024loraretriever}) & & & & \\
Compress-then-Serve~\citep{gabrielsson2025compress} & \checkmark & $1$ & static & $1\times$ \\
Trained MoE-LoRA (X-LoRA~\citep{buehler2024xlora}, MoLE~\citep{wu2024mole}, Arrow~\citep{ostapenko2024arrow}, PHATGOOSE~\citep{phatgoose})$^\dagger$ & \xmark & $N$ & dynamic (trained) & $\sim N\times$ \\
\midrule
\textbf{\LoRASpace{} (ours)}                         & \textbf{\checkmark} & $\boldsymbol{N}$ & \textbf{static + dynamic} & \textbf{$\sim 1\times$ static, sub-composite dynamic$^\ddagger$} \\
\bottomrule
\end{tabular}%
}
\\[2pt]
{
\scriptsize 
$^\dagger$Trained Mixture-of-Experts LoRA (MoE-LoRA) methods require joint LoRA+router training, unlike the post-hoc training-free setting of the other rows. $^\ddagger$Dynamic-mode uses one base-model pass; LoRA-side work is linear in active $N$ inside the $D$-dimensional subspace. Measured on Pool C: $2.3\times$ faster than composite at $N{=}2$, growing to $6.1\times$ at $N{=}6$ (\S\ref{sec:experiments:quality}).}
\end{table}

\paragraph{Where \LoRASpace{} sits.}
Existing training-free multi-LoRA methods trade off cost against compositionality. Static merges and shared-basis merges collapse the active subset into one delta at ambient cost ($1\times$), but freeze the composition before inference. Composite, step/mask routing, and per-input router methods recover input-dependent routing at $\sim N\times$ cost, because each active LoRA runs as its own forward pass. The closest neighbour, Compress-then-Serve~\citep{gabrielsson2025compress}, introduces a pool-wide shared basis but targets single-LoRA serving ($N{=}1$) rather than $N\ge 2$ composition. Tab.~\ref{tab:design_space} arranges these methods along four axes (pool-wide basis, active LoRAs per step, static vs.\ dynamic, per-image cost) and identifies the empty slot. No reviewed method combines a pool-wide basis with simultaneous $N\ge 2$ composition and both deployment modes (defined in \S\ref{sec:method:operator}) from one precompute. \LoRASpace{} fills it. The decomposition itself is standard. What is new is applying it once to a whole pool, so that every adapter keeps its own coordinates in one shared space and any combination is served in one pass from one compact object. Two further capabilities separate it from the closest pool-level methods. $\rho$ predicts before deployment whether a pool compresses (\S\ref{sec:experiments:overlap}), and a newly published adapter can be admitted by extending the bases without recomputing the pool (App.~\ref{app:dynamic_pools}). App.~\ref{app:capabilities} compares the methods on these requirements. We exercise three adaptive routers (spatial~\citep{fleshman2025spectr}, temporal~\citep{zou2025cmlora}, semantic~\citep{freefuse2025}) plus a uniform-merge baseline~\citep{ilharco2023taskarith}, all adapted from existing designs.

\paragraph{Static merge.}
Static merge methods fold a chosen subset of LoRAs into a single ambient delta before inference. Task-arithmetic~\citep{ilharco2023taskarith} averages deltas, while TIES-Merging~\citep{yadav2023ties} resolves sign conflicts and DARE~\citep{yu2024dare} adds drop-and-rescale delta sparsification. The shared property is that all active LoRAs become one low-rank update, so inference matches a single-LoRA model in cost. Quality \emph{drops} on multi-LoRA tasks~\citep{zhong2024composlora}, which is usually attributed to interference between LoRAs trained for different concepts when their deltas are summed.

\paragraph{Composite and route-then-apply.}
Composite methods avoid this interference by keeping LoRAs \emph{separate} at inference. ComposLoRA~\citep{zhong2024composlora} runs one full forward pass per active LoRA, paying linearly in compute and memory with $N$. Adaptive selection reduces cost through dynamic routing: CMLoRA~\citep{zou2025cmlora} and MultLFG~\citep{multlfg2025} schedule dominant LoRAs across diffusion steps via frequency-domain signals; FreeFuse~\citep{freefuse2025} routes subject tokens to spatial regions at test time; ZipLoRA~\citep{shah2024ziplora}, K-LoRA~\citep{klora2025}, and FreeLoRA~\citep{freelora2025} target narrower 2-LoRA / autoregressive personalization settings. None of these expose a pool-wide precomputed substrate handling arbitrary multi-LoRA composition with both static and dynamic deployment. Orthogonality-based merge analyses such as OMCD~\citep{omcd2025} are complementary: they study inter-LoRA orthogonality as an explanatory variable for merge quality, where our framework treats pool-wide rank-fill as the structural variable.

\paragraph{Subspace methods.}
A separate line of work compresses LoRA representations through linear projection. Knots~\citep{stoica2025knots} column-stacks per-LoRA deltas $\Delta W_i = B_i A_i$ and applies one SVD as an alignment tool before TIES-style merging on small LoRA sets. CoreSpace~\citep{corespace2025} stacks per-LoRA factors and aligns them in a fixed $\sum_i r_i$ basis without truncation. Reported evaluations are limited to small-pool settings, so we do not include its numbers in our pool-scale comparison. Compress-then-Serve~\citep{gabrielsson2025compress} learns a pool-wide shared basis $(U, V)$ together with per-LoRA middle matrices $\Sigma_i$ to compress the joint product $B_i A_i$ (the per-LoRA delta), for efficient serving of many single-LoRA requests on LLM backbones. Systems such as Punica~\citep{chen2024punica} and SLoRA~\citep{sheng2024slora} optimise serving many LoRA-specialised requests, but do not address simultaneous semantic composition of multiple LoRAs in one generation. CtS notes after its Proposition~1 that training noise makes the joint rank generically full ($\tilde r = \sum_i r_i$); our energy-truncated measurement of $D$ at $\tau{=}0.99$ is consistent with this high-effective-rank regime on the LLM pool ($\rho{=}0.96$) but substantially smaller on diffusion pools ($\rho \in [0.46, 0.57]$, Tab.~\ref{tab:rank_fill}). \LoRASpace{} differs from CtS in three respects: (i) the inference setting is simultaneous $N{\ge}2$ composition rather than single-LoRA serving; (ii) we use paired marginal SVDs to factorise the energy-normalised $B$ and $A$ stacks separately rather than the joint product; (iii) the substrate exposes a router-compatible operator $U(\sum_i w_i \hat{B}_i \hat{A}_i)V^\top x$ (defined in \S\ref{sec:method:operator}) supporting both static and dynamic deployment from one precompute.

\paragraph{Single-active routing.}
Switch-style methods activate exactly one LoRA at any token or step (SpectR~\citep{fleshman2025spectr}, X-LoRA~\citep{buehler2024xlora}, PHATGOOSE~\citep{phatgoose}, MoLE~\citep{wu2024mole}, Arrow~\citep{ostapenko2024arrow}). They pick one LoRA at each location, so the multi-LoRA combination question never arises and they are not directly comparable to composition methods. We report a switch baseline in App.~\ref{app:switch_comparison} for completeness.

\paragraph{LLM-side composition.}
LoRA-Retriever~\citep{zhao2024loraretriever} composes task-specialised LoRAs for LLMs by selecting a top-$K$ subset via INSTRUCTOR~\citep{su2023instructor} embeddings, then either fusing factors ($B_p A_q$ cross-terms) or running a weighted ambient mixture. LoRAHub~\citep{huang2024lorahub} optimises LoRA weights with CMA-ES~\citep{hansen2001cma} on a held-out validation set. Both treat composition as ambient-space averaging; neither exploits pool-level rank structure. \LoRASpace{}'s operator $\sum_i w_i \hat{B}_i \hat{A}_i$ contains only per-LoRA products by construction, so the cross-product approximation error that degrades fusion-style factor combination (App.~\ref{app:llm_full}) does not arise. We use the LoRA-Retriever pool to characterise the high-effective-rank regime as a counterpoint to the diffusion regime.

\section{Method}
\label{sec:method}
A LoRA $i$ (e.g., an image style LoRA) comprises two matrices $A_i$ and $B_i$ for each layer. In general, as indicated later in \S\ref{sec:method:setup}, a scaled linear combination of $A_iB_i$ pairs produces the desired delta. In \LoRASpace{}, we produce these deltas by pre-computing a decomposition (construction \S~\ref{sec:method:basis}) of the stacked $A_i$s and $B_i$s and using their sub-space representations (low-dimension) $\hat{A}_i$ and $\hat{B}_i$ appropriately projected. As we control the rank of this decomposition, and hence the number of actual deltas, the computational complexity drastically reduces, as we show in the following sections. 

\begin{figure}[!htbp]
\centering
\tikzset{
  bowtie/.pic={
    \tikzset{bowtie/.cd,#1}
    \def\w{\pgfkeysvalueof{/tikz/bowtie/width}}
    \def\h{\pgfkeysvalueof{/tikz/bowtie/height}}
    \def\d{\pgfkeysvalueof{/tikz/bowtie/depth}}
    \def\col{\pgfkeysvalueof{/tikz/bowtie/color}}    

    \coordinate (A) at (-\w/2,-\h/2);
    \coordinate (B) at (-\w/2,\h/2);
    \coordinate (C) at (0, \d/2) ;
    \coordinate (D) at (0, -\d/2);
    \coordinate (E) at (\w/2, \h/2);
    \coordinate (F) at (\w/2, -\h/2);

    \fill[\pgfkeysvalueof{/tikz/bowtie/fill1}] (A)--(B)--(C)--(D)--cycle;
    \fill[\pgfkeysvalueof{/tikz/bowtie/fill2}] (D)--(C)--(E)--(F) --cycle;

    \draw[\pgfkeysvalueof{/tikz/bowtie/draw}]
      (A)--(B)--(C)--(D)--cycle
      (D) -- (F)--(E) -- (C) -- cycle;
      
      \coordinate (-center) at (0,0);
       \coordinate (-nw) at (B);
        \coordinate (-se) at (F);
        \coordinate (-e) at ($(F)!0.5!(E)$);
        \coordinate (-w) at ($(A)!0.5!(B)$);
  },
  bowtie/.cd,
  width/.initial=2,
  height/.initial=1.2,
  fill1/.initial=orange!30,
  fill2/.initial=orange!20,
  draw/.initial=\col!80!black,
  depth/.initial=1,
  color/.initial=black
 
}

\begin{tikzpicture}
  
  \pic (girl) at (0,0) {bowtie={width=1.6, height=1, depth=0.4, color=orange, fill1=orange!30, fill2=orange!10}} ;
  \node (gl) at (girl-center) [] {Girl};
  \node at (gl) [anchor=north, yshift = -0.35cm] {$+$};
  
    \pic (forest) at (0,-1.2) {bowtie={width=1.6, height=1, depth=0.4, color=green, fill1=green!30, fill2=green!10}} ;
  \node (fl) at (forest-center) [] {Forest};
    \node at (fl) [anchor=north, yshift = -0.35cm] {$+$};
  
    \pic (angry) at (0,-2.4) {bowtie={width=1.6, height=1, depth=0.4, color=blue, fill1=blue!30, fill2=blue!10}} ;
  \node (al) at (angry-center) [] {Angry};  
  
  \draw[dashed] ($(girl-nw) + (-0.1, 0.1)$) rectangle ($(angry-se) + (0.1, -0.1)$);
   \pic (merge) at ($(forest-center) + (2,0)$) {bowtie={width=1.6, height=1, depth=0.4, color=purple, fill1=purple!30, fill2=purple!10}} ;
   
   \draw[-latex] ($(forest-e)+(0.1, 0)$) -- (merge-w);
   
    \node at ($(merge-center) + (0, -1.5)$) (mergelabel) [text width=2cm, align=center] {Merge by summing weights};
   
   \node at ($(angry-center) + (0, -1)$) (mergelabel) {Merge};
    \node at ($(mergelabel) + (4.8, 0)$) (compositelabel) {Composite};
  \node at ($(mergelabel) + (10, 0)$) (loraspacel) {LoRASpace (ours)};

   \begin{scope}[xshift=4.8cm]
	  \pic (girl) at (0,0) {bowtie={width=1.6, height=1, depth=0.4, color=orange, fill1=orange!30, fill2=orange!10}} ;
	  \node (gl) at (girl-center) [] {Girl};

	    \pic (forest) at (0,-1.2) {bowtie={width=1.6, height=1, depth=0.4, color=green, fill1=green!30, fill2=green!10}} ;
	  \node (fl) at (forest-center) [] {Forest};

	    \pic (angry) at (0,-2.4) {bowtie={width=1.6, height=1, depth=0.4, color=blue, fill1=blue!30, fill2=blue!10}} ;
	  \node (al) at (angry-center) [] {Angry};  
	  
	\coordinate (o) at (-1.5, -1.2);
	\coordinate (e) at (1.0, -1.2);
	\draw [-latex] (o) -- (forest-w);
	\draw [-latex] (o) ++ (0.5, 0) |- (girl-w);
	\draw [-latex] (o) ++ (0.5, 0) |- (angry-w);

	   \node (sum) at ($(forest-center) + (1.85,0)$) [draw, circle] {\scriptsize $\sum$} ;
       \node at (sum) [anchor= north, yshift=-0.5cm, text width=2cm, align=center]  {Compose: summing generated outputs};
	   
	   \draw[-latex] (forest-e) -- (sum);   
	   \draw[] (girl-e) -| (e);
	   \draw[] (angry-e) -| (e);
\end{scope}

   \begin{scope}[xshift=10cm]
	  \pic (girl) at (0,0) {bowtie={width=1.6, height=1, depth=0.4, color=orange, fill1=orange!30, fill2=orange!10}} ;
	  \node (gl) at (girl-center) [] {Girl proj.};

	    \pic (forest) at (0,-1.2) {bowtie={width=1.6, height=1, depth=0.4, color=green, fill1=green!30, fill2=green!10}} ;
	  \node (fl) at (forest-center) [] {Forest proj.};

	    \pic (angry) at (0,-2.4) {bowtie={width=1.6, height=1, depth=0.4, color=blue, fill1=blue!30, fill2=blue!10}} ;
	  \node (al) at (angry-center) [] {Angry proj.};  
	  
	\coordinate (o) at (-1.5, -1.2);
	\coordinate (e) at (1.0, -1.2);
	\draw [-latex] (o) -- (forest-w);
	\draw [-latex] (o) ++ (0.5, 0) |- (girl-w);
	\draw [-latex] (o) ++ (0.5, 0) |- (angry-w);

	   \node (sum) at ($(forest-center) + (1.75,0)$) [draw, circle] {\scriptsize $\sum$} ;
	   
	   \draw[-latex] (forest-e) -- (sum);   
	   \draw[] (girl-e) -| (e);
	   \draw[] (angry-e) -| (e);
	   \draw[-latex] (sum) -- ++(0.75, 0) coordinate (opbasis);
	   
	   \draw [fill=purple!30] (opbasis) -- ++(0, 1) -- ++(0.5, 0.3) -- ++(0, -2.6) -- ++(-0.5, 0.3) -- cycle;
	   \node at ($(opbasis) +(0.3, -1.1)$) [anchor=west, rotate=90, align=center] {Output basis};
	   
	    \draw [fill=purple!30] (-1.5,-1.2) -- ++(0, 1) -- ++(-0.5, 0.3) -- ++(0, -2.6) -- ++(0.5, 0.3) -- cycle;
	   \node at ($(-2.05, -1.0) +(0.3, -1.1)$) [anchor=west, rotate=90, align=center] {Input basis};	    
\end{scope}

\end{tikzpicture}
\caption{Three strategies for composing $N$ active LoRAs from a \emph{Combo} (the active subset for one generation). \emph{Merge} folds them into one delta with one forward pass. \emph{Composite} runs $N$ forward passes. \emph{\LoRASpace{}} projects via $V^\top$, mixes $\sum_i w_i \hat B_i \hat A_i$, and back-projects via $U$ in one pass. Static-mode pre-folds the inner sum when $\{w_i\}$ are time-invariant; dynamic-mode evaluates it per UNet pass (precompute illustration in App.~\ref{app:precompute_fig}).}
\label{fig:method_overview}
\end{figure}
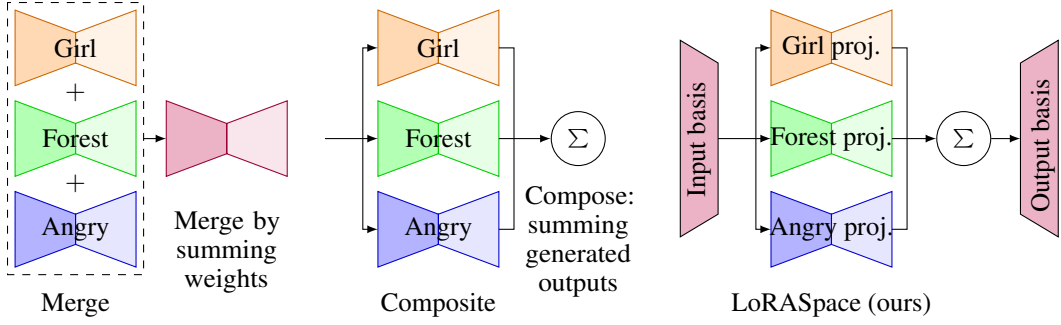

\subsection{Setup and notation}
\label{sec:method:setup}

Consider a single LoRA-modified linear layer with frozen base weight $W_0 \in \mathbb{R}^{\dout \times \din}$. A LoRA~\citep{hu2022lora} augments this layer with a low-rank update $\Delta W_i = s_i B_i A_i$, where $B_i \in \mathbb{R}^{\dout \times r_i}$ and $A_i \in \mathbb{R}^{r_i \times \din}$ are the LoRA's left and right factors, $r_i \ll \min(\din, \dout)$ is the LoRA rank, and $s_i = \alpha_i / r_i$ is the per-LoRA scaling factor with $\alpha_i$ the LoRA scaling hyperparameter (set at training time). We write $\bar r = (1/K)\sum_i r_i$ for the average pool rank. Given a pool $\mathcal{P} = \{1, \dots, K\}$ of trained LoRAs, the multi-LoRA composition target on input $x \in \mathbb{R}^{\din}$ is
\begin{equation}
\Delta y^{\star}(x) = \sum_{i=1}^{N} w_i\, s_i B_i A_i\, x,
\label{eq:target_multilora}
\end{equation}
where $\{w_i\}_{i=1}^{N}$ are per-LoRA weights for an active subset of size $N \le K$. Static merge precomputes a fixed sum once and discards routing. Composite evaluates Eq.~\eqref{eq:target_multilora} by running $N$ forward passes per token and aggregating outputs. \LoRASpace{} evaluates a router-weighted form of Eq.~\eqref{eq:target_multilora} in one base-model pass with LoRA-side work linear in $N$ inside a shared $D$-dimensional subspace ($D \ll \dout$), exploiting structural rank overlap.

\subsection{Energy-normalised shared subspace}
\label{sec:method:basis}

The motivating observation is empirical. When per-LoRA factors from a real diffusion pool are stacked and decomposed by paired marginal SVDs, the resulting subspace dimension is much smaller than $\sum_i r_i$. We characterise this empirically in \S\ref{sec:experiments:overlap}, which separates alignment between adapters from layer-width saturation and each adapter's own spectral tail. Here we use it to construct an energy-truncated shared subspace.

\paragraph{Frobenius normalisation.}
Raw factor stacking is dominated by high-norm LoRAs. Frobenius normalisation puts every LoRA on a common scale before decomposition:
\begin{equation}
\bar{B}_i = B_i / \|B_i\|_F, \qquad \bar{A}_i = A_i / \|A_i\|_F.
\label{eq:frobenius_norm}
\end{equation}
We then form pool stacks $\bar{B}_{\mathrm{stack}} = [\bar{B}_1 \cdots \bar{B}_K] \in \mathbb{R}^{\dout \times \sum_i r_i}$ and $\bar{A}_{\mathrm{stack}} = [\bar{A}_1; \cdots; \bar{A}_K] \in \mathbb{R}^{\sum_i r_i \times \din}$.

\paragraph{Paired marginal SVDs and adaptive truncation.}
We compute two thin SVDs $\bar{B}_{\mathrm{stack}} = U_B \Sigma_B V_B^\top$ and $\bar{A}_{\mathrm{stack}} = U_A \Sigma_A V_A^\top$, then truncate each at energy fraction $\tau \in (0, 1)$:
\begin{equation}
D_B = \min\Bigl\{d \;\Big|\; \frac{\sum_{j \le d}\sigma_j(\bar{B}_{\mathrm{stack}})^2}{\sum_{j}\sigma_j(\bar{B}_{\mathrm{stack}})^2} \ge \tau \Bigr\},
\quad D_A \text{ analogously.}
\label{eq:dim_selection}
\end{equation}
The shared subspace dimension is $D = \min(\max(D_B, D_A),\,\sum_i r_i,\,\din,\,\dout)$. The shared bases are $U = U_B[:,\, {:}D] \in \mathbb{R}^{\dout \times D}$ and $V = V_A[:,\, {:}D] \in \mathbb{R}^{\din \times D}$ (with $V_A$ the right-singular-vector matrix in $\bar A_{\mathrm{stack}} = U_A \Sigma_A V_A^\top$), the leading left and right singular vectors of $\bar B_{\mathrm{stack}}$ and $\bar A_{\mathrm{stack}}$, both with orthonormal columns. By the Eckart-Young theorem~\citep{eckart1936}, $U$ and $V$ are the optimal rank-$D$ subspaces for the marginal $B$-stack and $A$-stack reconstructions separately. This is optimal for each marginal factor projection but not a joint optimum for the product $\sum_i s_i B_i A_i$. A joint SVD of the scaled deltas is closer to product-optimal (per-LoRA error $0.35\%$ against $0.88\%$ on Pool C at equal $D$), but the difference does not reach the generated images, and the marginal bases are built from thin factors on CPU and can be extended when the pool changes (App.~\ref{app:joint_basis}). We find empirically that retaining the low-energy directions does not improve composition metrics within the practical $\tau$ range (\S\ref{sec:experiments:energy}).

\subsection{Per-LoRA projected factors}
\label{sec:method:factors}

Each LoRA is re-parameterised by projecting its raw factors onto the shared subspace:
\begin{equation}
\hat{B}_i = \sqrt{s_i}\; U^\top B_i \in \mathbb{R}^{D \times r_i},
\qquad
\hat{A}_i = \sqrt{s_i}\; A_i V \in \mathbb{R}^{r_i \times D}.
\label{eq:projected_factors}
\end{equation}
The LoRA scaling $s_i$ is split symmetrically as $\sqrt{s_i}$ on each side. The per-layer cache $\mathcal{C}_\ell = \{U, V, \{(\hat{B}_i, \hat{A}_i)\}_{i=1}^{K}\}$ is computed once per pool, stored on CPU memory, and streamed to the GPU as needed. Cache size is $O(D(\dout + \din) + K D \bar{r})$ versus $O(K \bar{r}(\dout + \din))$ for ambient storage, which gives a meaningful reduction whenever $D \ll \dout$.

\subsection{Composition operator and deployment modes}
\label{sec:method:operator}

At inference, the \LoRASpace{} operator processes input $x$ in three steps. Project once into the subspace $x_{\mathrm{proj}} = V^\top x \in \mathbb{R}^D$. Compute a per-LoRA subspace contribution $y_i = \hat{B}_i (\hat{A}_i x_{\mathrm{proj}}) \in \mathbb{R}^D$ for each active $i$. Mix and back-project to give $\Delta y^{\mathrm{LS}}(x) = U \sum_{i=1}^{N} w_i y_i$. Compactly,
\begin{equation}
\Delta y^{\mathrm{LS}}(x)
= U \!\left(\sum_{i=1}^{N} w_i\, \hat{B}_i \hat{A}_i\right) V^\top x.
\label{eq:loraspace_operator}
\end{equation}
The input projection $V^\top x$ and output back-projection $U(\cdot)$ are shared across all active LoRAs. The per-LoRA work is entirely inside the $D$-dimensional subspace. The shared bases do not orthogonalise adapters or split directions among them. $(U, V)$ spans the union of the pool's directions and each adapter keeps dense coordinates inside it. Across $4{,}800$ audited axes ($12$ SDXL layers, both sides, top $200$ per side) the median adapter-ownership Herfindahl index is $0.12$, so a typical axis carries about eight adapters. At full rank the operator equals the ambient weighted sum (Prop.~\ref{prop:exact_recovery}), so any interaction the ambient sum produces, such as a lighting LoRA changing how a character LoRA renders, is kept. Truncation only removes a low-energy tail (composed $N{=}5$ error $0.84\%$ at $\tau = 0.99$, App.~\ref{app:joint_basis}).

\paragraph{Static and dynamic deployment from one substrate.} When $\{w_i\}$ are time-invariant (uniform or per-prompt/pre-defined routing), the inner sum $M = \sum_i w_i \hat{B}_i \hat{A}_i \in \mathbb{R}^{D \times D}$ folds into a single ambient delta $\Delta = U M V^\top$ at ambient-merge cost (static-mode). Time-varying routers (per-token, per-step, per-pixel) require dynamic-mode. Eq.~\eqref{eq:loraspace_operator} is then evaluated in the subspace at every UNet pass, with shared $V^\top x$ and $U(\cdot)$ paid once per pass independent of $N$. The same $(U, V)$ serves both modes.

\paragraph{Static-mode is not ambient merge.}
Substituting Eq.~\eqref{eq:projected_factors} gives $\Delta y^{\mathrm{LS}}(x) = P_U \, \Delta y^{\star}(P_V x)$ with $P_U = UU^\top$, $P_V = VV^\top$ the projectors onto the truncated subspaces. Since $D < \min(\dout, \din)$ on most layers ($682$ of $722$ on Pool C), $P_U, P_V \neq I$ there and the two operators differ from the unprojected target $\Delta y^{\star}$. \S\ref{sec:experiments:quality} reports the empirical effect alongside a mechanism ablation.

\subsection{Reconstruction guarantee}
\label{sec:method:theory}

\begin{proposition}[Exact recovery]
\label{prop:exact_recovery}
If $\mathrm{col}(B_i) \subseteq \mathrm{col}(U)$ and $\mathrm{row}(A_i) \subseteq \mathrm{col}(V)$ for all $i \in \mathcal{P}$, then $\Delta y^{\mathrm{LS}}(x) = \Delta y^{\star}(x)$ for any weights $\{w_i\}$.
\end{proposition}

\begin{corollary}[Truncation error bound]
\label{cor:truncation_bound}
Let $\sigma^B_{D+1}, \sigma^A_{D+1}$ be the leading discarded singular values of $\bar{B}_{\mathrm{stack}}, \bar{A}_{\mathrm{stack}}$, and $\mu^B = \max_i \|B_i\|_F$, $\mu^A = \max_i \|A_i\|_F$. Then for any $x$ and $\{w_i\}$,
\begin{equation}
\|\Delta y^{\mathrm{LS}}(x) - \Delta y^{\star}(x)\|_2
\le \|x\|_2 \cdot \mu^A \mu^B \cdot (\sigma^B_{D+1} + \sigma^A_{D+1} + \sigma^B_{D+1}\sigma^A_{D+1}) \cdot \sum_i s_i |w_i|.
\label{eq:truncation_bound}
\end{equation}
\end{corollary}

Proofs are in Appendix~\ref{app:proof_prop1}. At $\tau = 0.99$ the per-side discarded singular value is bounded by $\sqrt{(1-\tau)K} \le 0.41$ for $K \le 17$; \S\ref{sec:experiments:overlap} confirms this approximation does not degrade outputs in practice.

\subsection{Router interface}
\label{sec:method:routing}

The composition operator is router-pluggable. Any function producing per-LoRA weights $\{w_i\}$ plugs in unchanged. We adopt uniform weights ($w_i = 1/N$) as the default and characterise three adaptive routers as plug-in extensions: spatial (per-token activation energy, in the spirit of SpectR~\citep{fleshman2025spectr}), temporal (per-step subspace-band, following CMLoRA~\citep{zou2025cmlora}), and semantic (per-concept token-to-region routing, in the spirit of FreeFuse~\citep{freefuse2025}). Full specifications are in App.~\ref{app:router_specs}. For LLM pools, \LoRASpace{} pairs with an INSTRUCTOR retriever~\citep{su2023instructor} following the LoRA-Retriever protocol~\citep{zhao2024loraretriever}, showing that internal and external routing policies share a single composition substrate.

\subsection{Complexity}
\label{sec:method:complexity}

We count the incremental LoRA-side linear algebra for one layer processing one token; Appendix~\ref{app:complexity} reports the full per-step cost including base-model forwards. For $N$ active LoRAs of representative rank $\bar{r}$, ambient composition adds $O(N \bar{r}(\din + \dout))$ per layer, linear in both $N$ and ambient width. \LoRASpace{} costs
\begin{equation}
\mathrm{cost}_{\mathrm{LS}}
= \underbrace{O\!\bigl(D(\din + \dout)\bigr)}_{\text{shared projection}}
\;+\; \underbrace{O\!\bigl(N \bar{r} D\bigr)}_{\text{per-LoRA subspace paths}}.
\label{eq:complexity}
\end{equation}
The shared term is paid once regardless of $N$. As pool size $K$ grows, the projection size scales as $D \propto K^{\beta}$ with measured $\beta \in [0.43, 0.55]$ on diffusion pools (\S\ref{sec:experiments:overlap}), so the projection term grows sub-linearly with the pool. The per-LoRA term is linear in $N$ but scaled by $D$, which is much smaller than the ambient $\dout$. The adapter-side crossover with an optimised single-pass ambient mixture is $D < 2 N \bar{r} d / (2d + N \bar{r})$. This bound is always below $N \bar{r}$, so it is not met at the default $\tau = 0.99$ (Pool C, mean $D = 819$ against $N \bar{r} \approx 527$ at $N{=}5$). It is met at the compressed operating points validated in App.~\ref{app:joint_basis} ($\bar{D} \approx 107$ at $\tau = 0.70$ and $\approx 49$ at $\tau = 0.50$), so the subspace path saves adapter-side arithmetic only in that reduced-storage regime. Appendix~\ref{app:complexity} gives the full per-method comparison including memory.

\section{Experiments}
\label{sec:experiments}

We organise the evaluation around four findings: rank-fill regime characterisation across diffusion and LLM pools (\S\ref{sec:experiments:overlap}); a quality + mechanism comparison showing static-mode \LoRASpace{} delivers composite-level quality at ambient-merge cost and that the lift is rank-$D$ truncation in any basis (\S\ref{sec:experiments:quality}); robustness to the energy truncation threshold (\S\ref{sec:experiments:energy}); and a routing ablation showing router quality varies on Pool A (temporal helps, semantic is flat, spatial regresses on ImageReward; \S\ref{sec:experiments:routing}). A predictive-theory check on the near rank-orthogonal LLM regime, including external retrieval coupling, is in \S\ref{sec:experiments:llm} (App.~\ref{app:llm_full} for the full scoreboard).

\subsection{Setup}
\label{sec:experiments:setup}

\paragraph{Pools and backbones.} Pool A (11 reality-style LoRAs, SD\,1.5~\citep{rombach2022sd}) and Pool B (17 LoRAs, SD\,1.5) are drawn from CivitAI~\citep{civitai2023}. Pool C (17 reality-style LoRAs, SDXL~\citep{podell2024sdxl}) is curated from CivitAI and Hugging Face~\citep{huggingface2016}. The LLM pool is the 48-LoRA LLaMA-2-7B~\citep{touvron2023llama2} task-LoRA set from LoRA-Retriever~\citep{zhao2024loraretriever}.

\paragraph{Metrics.} For diffusion, we report CLIP-T~\citep{radford2021clip} for prompt alignment, ImageReward~\citep{xu2023imagereward} for human-preference alignment, and a GPT-4o~\citep{openai2024gpt4o} composition rubric (Composition Quality $+$ Image Quality, $0$ to $10$ each, position-bias averaged across two orderings) following ComposLoRA~\citep{zhong2024composlora}. For LLMs, we use the LoRA-Retriever Macro Task accuracy.

\paragraph{Software and hardware.} All diffusion experiments run on a single RTX\,3090 (24\,GB) using PyTorch~\citep{paszke2019pytorch} and the \texttt{diffusers} library~\citep{diffusers2022}, with DPM-Solver++~\citep{lu2025dpmsolverpp} as the sampler. Pool A: 100 denoising steps, $512 \times 512$. Pool C: 50 denoising steps, $1024 \times 1024$. We use 240 stratified combinations per seed for Pool A, distributed across active sizes $N \in \{2, 3, 4, 5\}$, and 200 for Pool C, across $N \in \{2, \dots, 6\}$. Pool B is reported in the appendix at $N \in \{7, 8\}$ (Appendix~\ref{app:pool_b}) to probe the high-cardinality regime where the rank-fill ratio is lowest ($\rho = 0.47$ at $K = 17$); Pool A and Pool C cover the full active-LoRA range.

\subsection{Subspace dimension grows sub-linearly with pool size}
\label{sec:experiments:overlap}

We measure the rank-fill ratio $\rho = D / \sum_i r_i$ at energy threshold $\tau = 0.99$ (the value used throughout) and the growth exponent $\beta$ obtained by fitting $D \propto K^{\beta}$ with category-aware random subsampling at three random seeds. Table~\ref{tab:rank_fill} summarises the result. The three diffusion pools have $\rho$ between $0.46$ and $0.57$ and $\beta$ between $0.43$ and $0.55$, both far below the $\rho = 1$, $\beta = 1$ no-overlap baseline. The two LLM pools span the full range: a 48-LoRA sub-sample of LoRA-Hub over related FLAN-T5 NLP tasks lands at $\rho = 0.62$, $\beta = 0.88$, comparable to diffusion overlap, while the LoRA-Retriever benchmark, curated for retrieval-task diversity, sits near the no-overlap baseline at $\rho = 0.96$, $\beta = 0.94$.

\begin{table}[h]
\centering
\small
\caption{Rank overlap across pools at $\tau = 0.99$. $\rho = D / \sum_i r_i$ (lower means more overlap). $\beta$ is the power-law exponent of $D$ vs.\ $K$. The LoRA-Hub row is a $K{=}48$ sub-sample matched in size to LoRA-Retriever (full 196-LoRA pool in App.~\ref{app:llm_d_ablation}).}
\label{tab:rank_fill}
\begin{tabular}{lrrrrc}
\toprule
Pool & $K$ & $\sum_i r_i$ & $D$ (mean) & $\rho$ & $\beta$ \\
\midrule
Pool A (SD\,1.5)                           & 11 &  784 & 446 & 0.57 & 0.43 \\
Pool B (SD\,1.5)                           & 17 & 1112 & 521 & 0.47 & 0.49 \\
Pool C (SDXL)                              & 17 & 1792 & 819 & 0.46 & 0.55 \\
\midrule
LoRA-Hub (FLAN-T5-Large)                   & 48 &  768 & 474 & 0.62 & 0.88 \\
LoRA-Retriever (LLaMA-2-7B)                & 48 &  384 & 368 & 0.96 & 0.94 \\
\bottomrule
\end{tabular}
\end{table}

The source of the low $\rho$ differs between regimes. We split $\rho = \rho_{\mathrm{intra}} \times \rho_{\mathrm{cross}}$ into within-LoRA rank slack and a cross-LoRA term, and compare $D$ with a spectrum-matched random null (App.~\ref{app:rho_decomposition}). Within-LoRA slack is nowhere dominant ($\rho_{\mathrm{intra}} \ge 0.84$ on every pool). The related-task LoRA-Hub pool~\citep{huang2024lorahub} needs about $30\%$ fewer dimensions than its matched null, an alignment associated with shared task structure, while the diversity-curated LoRA-Retriever sits at its null. On the diffusion pools the stacked ranks exceed the layer width on most layers, and on the SDXL pools $D$ sits at its null ($0.98$ on Pool C), so their low $\rho$ is real compressibility that is largely explained by width saturation and per-adapter spectra, and we leave its link to training data open. Pool composition, not modality, predicts where compression is large and where it is not. The framework also predicts a saturation regime when $\sum_i r_i \ge d$: the full LoRA-Hub pool ($K{=}196$) crosses the threshold at $K^\star = d/r = 64$, with each new LoRA adding ${\sim}11$ dimensions to $\bar{D}$ below threshold and ${\sim}1$ above (App.~\ref{app:llm_d_ablation}).

\subsection{Quality, speed, and memory}
\label{sec:experiments:quality}

We compare three composition strategies on Pool A: ambient static merge with weight $1$ per LoRA, so the deltas are summed, ambient composite with $N$ forward passes per step, and \LoRASpace{} with uniform $1/N$ routing weights. Table~\ref{tab:main_quality} reports CLIP-T and ImageReward averaged across three random seeds ($240$ combinations per seed, $N \in \{2,3,4,5\}$), and the GPT-4o rubric from one seed (the position-bias-averaged GPT-4o evaluation is expensive; three-seed consistency on CLIP-T and ImageReward suffices to verify robustness).

\begin{table}[h]
\centering
\small
\caption{Pool A quality (mean $\pm$ std over 3 seeds, $240$ combos/seed, averaged over $N \in \{2, \dots, 5\}$, with GPT-4o from one seed). Higher is better for all metrics.}
\label{tab:main_quality}
\begin{tabular}{lrrrr}
\toprule
Method & CLIP-T $\uparrow$ & ImageReward $\uparrow$ & GPT-4o Total $\uparrow$ & GPT-4o Win\% $\uparrow$ \\
\midrule
merge                  & $33.60 \pm 0.19$ & $0.760 \pm 0.008$ & 14.32 & 21.9 \\
composite              & $35.27 \pm 0.19$ & $0.953 \pm 0.011$ & 15.24 & 29.5 \\
\LoRASpace{} (uniform) & $\mathbf{36.01 \pm 0.16}$ & $\mathbf{1.041 \pm 0.035}$ & \textbf{15.61} & \textbf{48.6} \\
\bottomrule
\end{tabular}
\end{table}

\LoRASpace{} with uniform routing exceeds composite by $+0.74$ CLIP-T, $+0.09$ ImageReward, $+0.37$ GPT-4o total, and almost twenty percentage points of GPT-4o win rate ($48.6\%$ vs.\ $29.5\%$, with $33.3\%$ being the no-preference baseline for three competitors). It also exceeds ambient merge by larger margins. The inter-method gap exceeds the across-seed standard deviation on every metric.

\begin{figure}[!htbp]
\centering
\includegraphics[width=0.38\linewidth]{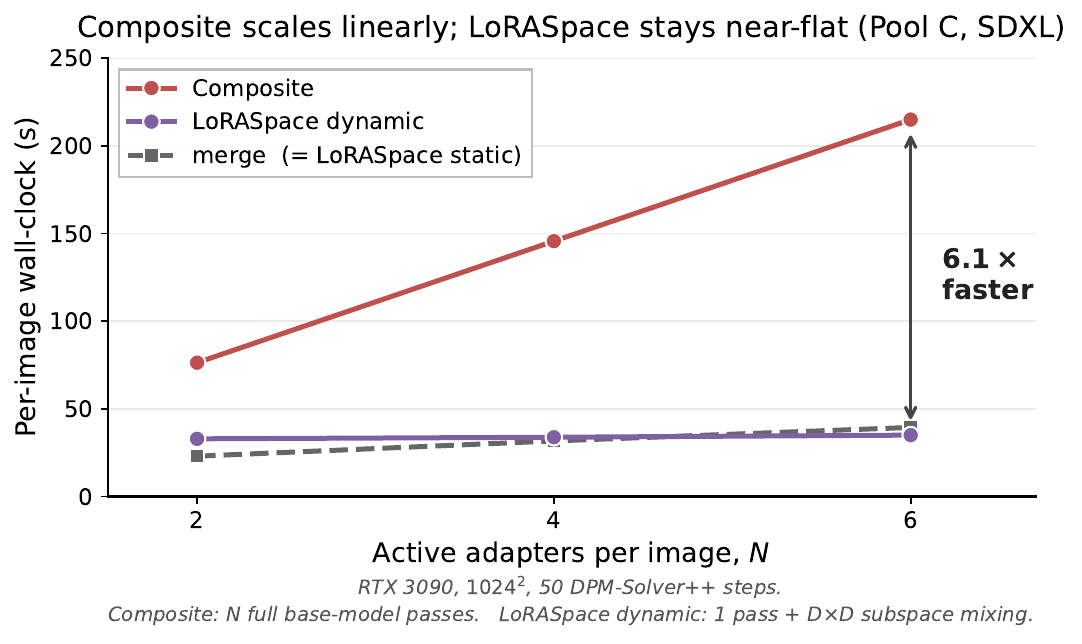}
\hfill
\includegraphics[width=0.38\linewidth]{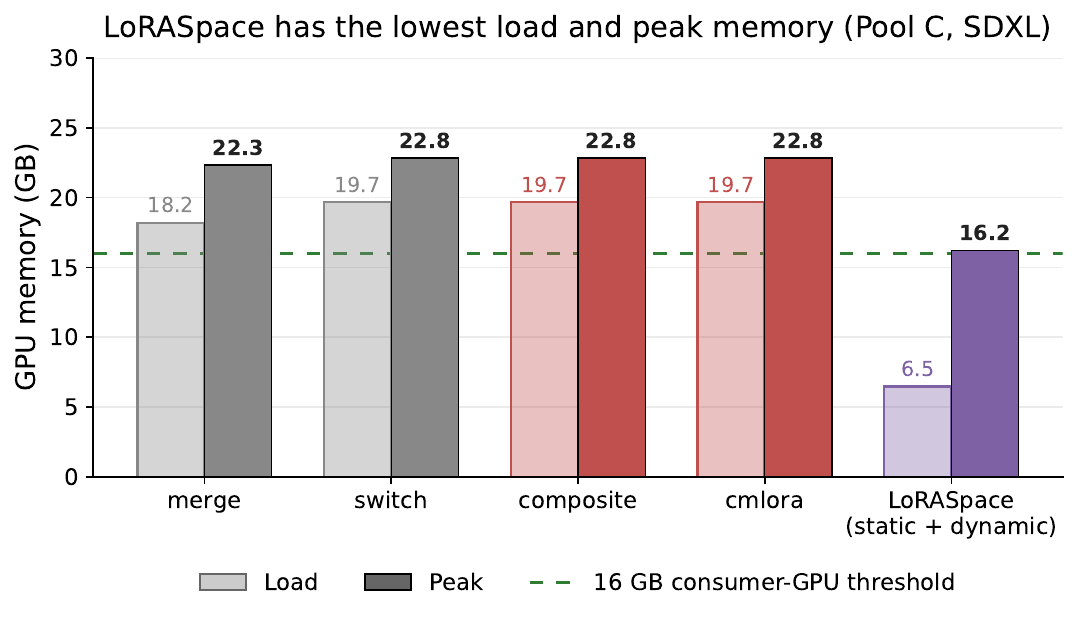}
\caption{\textbf{Pool C (SDXL) speed and memory.} \textit{Left:} per-image wall-clock vs.\ $N$. Composite scales linearly. Dynamic \LoRASpace{} stays near-flat (\textbf{$6.1\times$ speedup at $N{=}6$}). Static \LoRASpace{} inherits ambient-merge cost. \textit{Right:} GPU memory. \LoRASpace{} has the lowest load and peak memory ($6.5$/$16.2$\,GB). PEFT baselines (merge, switch, composite, cmlora) exceed $22$\,GB peak. RTX\,3090, $1024^2$, 50 DPM-Solver++ steps.}
\label{fig:speed_memory}
\end{figure}

\paragraph{Mechanism ablation.}
We test whether the quality lift depends on the pool-wide basis with two static-delta variants at the same rank $D$. \texttt{truncated-merge} is the combo-specific best-fit SVD of the ambient merge; \texttt{subspace-merge} bakes the pool-wide $(U, V)$ readout into a fixed delta per combo. Both run at ambient-merge per-image cost. On Pool A (seed $111$, $240$ combos, $N{=}2$ to $5$), both match per-step \LoRASpace{}-uniform within fp16 noise (Tab.~\ref{tab:mechanism_app}). At equal rank $D$ the basis choice does not affect quality. The lift over ambient merge ($+2.40$ CLIP-T, $+0.32$ ImgRew) is therefore not a property of the pool-wide basis. Both variants average the LoRAs with weights $1/N$ where merge sums them, and at $\tau = 0.99$ the truncation changes the composed delta by under $1\%$ on Pool C (App.~\ref{app:joint_basis}).

\paragraph{Why the pool-wide basis.}
\emph{Amortisation}: \texttt{truncated-merge} requires $192$ layer SVDs per combo (${\sim}1.4$\,h across Pool A's $1011$ combos) vs.\ one $21.8$\,s pool-wide precompute (App.~\ref{app:pool_c_speed}). \emph{Dynamic-mode}: when $\{w_i\}$ depend on context, no single static delta exists for \texttt{truncated-merge}, while $(U, V)$ supports both modes from one precompute. At about $5$\,s of SVDs per combination, the precompute pays off beyond about five distinct combinations, so per-combination truncated merge remains the better choice for a static deployment with a few fixed combinations. Serving $30$ combinations from the $37$-adapter union of App.~\ref{app:heterogeneous} averages $35.1$\,s per combination including the one-time load, against $43.9$\,s for repeated PEFT load and merge. Keeping all adapters resident instead needs per-adapter format fixes, and even then $36$ of the $37$ load at $19.0$\,GiB and one $1024^2$ generation peaks at $23.1$ of the card's $24$\,GiB.

\paragraph{Pool C (speed, memory, quality).}
Fig.~\ref{fig:speed_memory} summarises per-image cost. Dynamic-mode stays near-flat in $N$ ($6.1\times$ faster than composite at $N{=}6$). Most of this gain comes from running one backbone pass instead of $N$. At $N{=}5$ a single-pass ambient mixture already accounts for $4.9\times$ of the $5.3\times$, and the substrate path adds $1.08\times$ (Tab.~\ref{tab:latency_split}). Static-mode inherits ambient merge cost ($23$ to $39$\,s, Tab.~\ref{tab:speed_memory_full}). Across $200$ combos uniform \LoRASpace{} scores $\mathbf{32.89}$ CLIP-T / $\mathbf{1.040}$ ImgRew vs.\ composite ($31.91$ / $0.973$) and merge ($29.41$ / $0.721$), with the CLIP-T gap over composite growing from $+0.17$ at $N{=}2$ to $+1.37$ at $N{=}6$ (Tab.~\ref{tab:pool_c_quality}, additional qualitative examples in App.~\ref{app:qualitative_extra}).

\begin{figure}[t]
\centering
\includegraphics[width=0.21\linewidth]{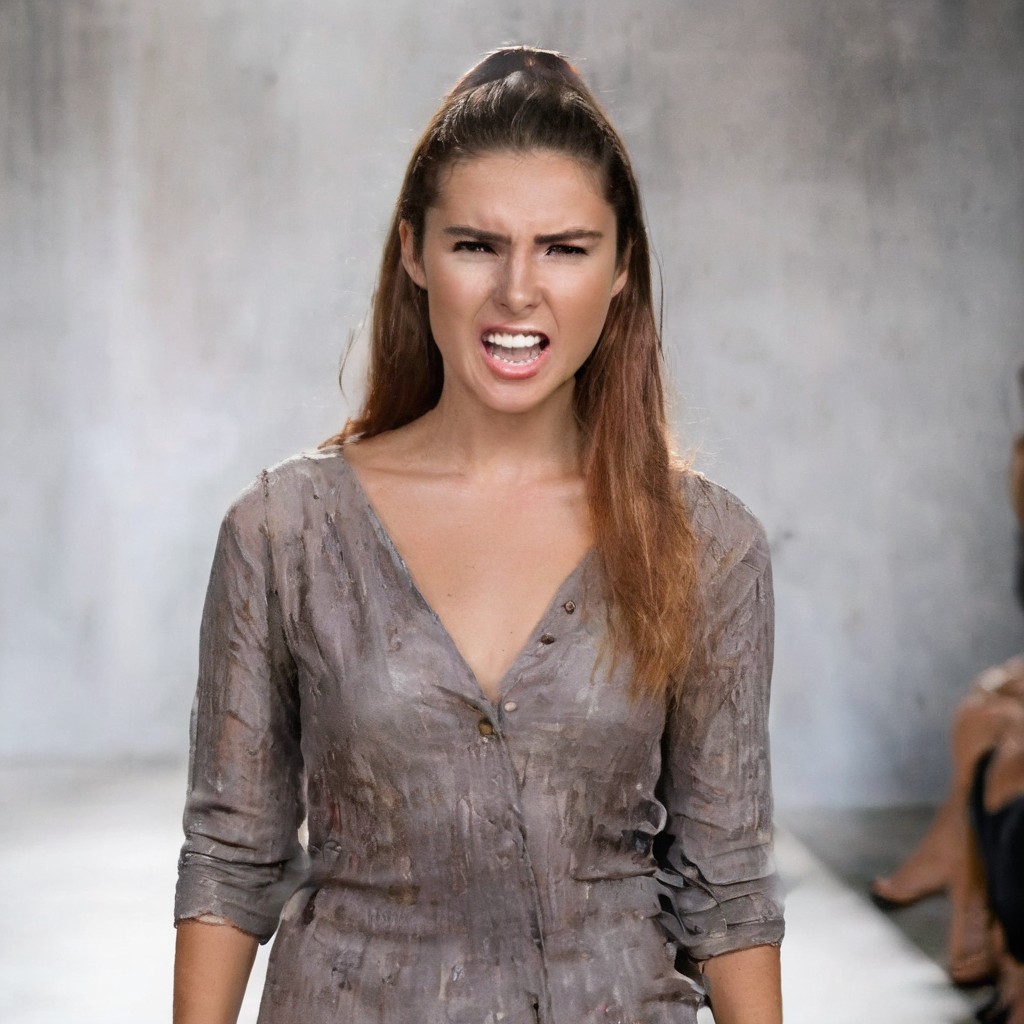}\hspace{6pt}%
\includegraphics[width=0.21\linewidth]{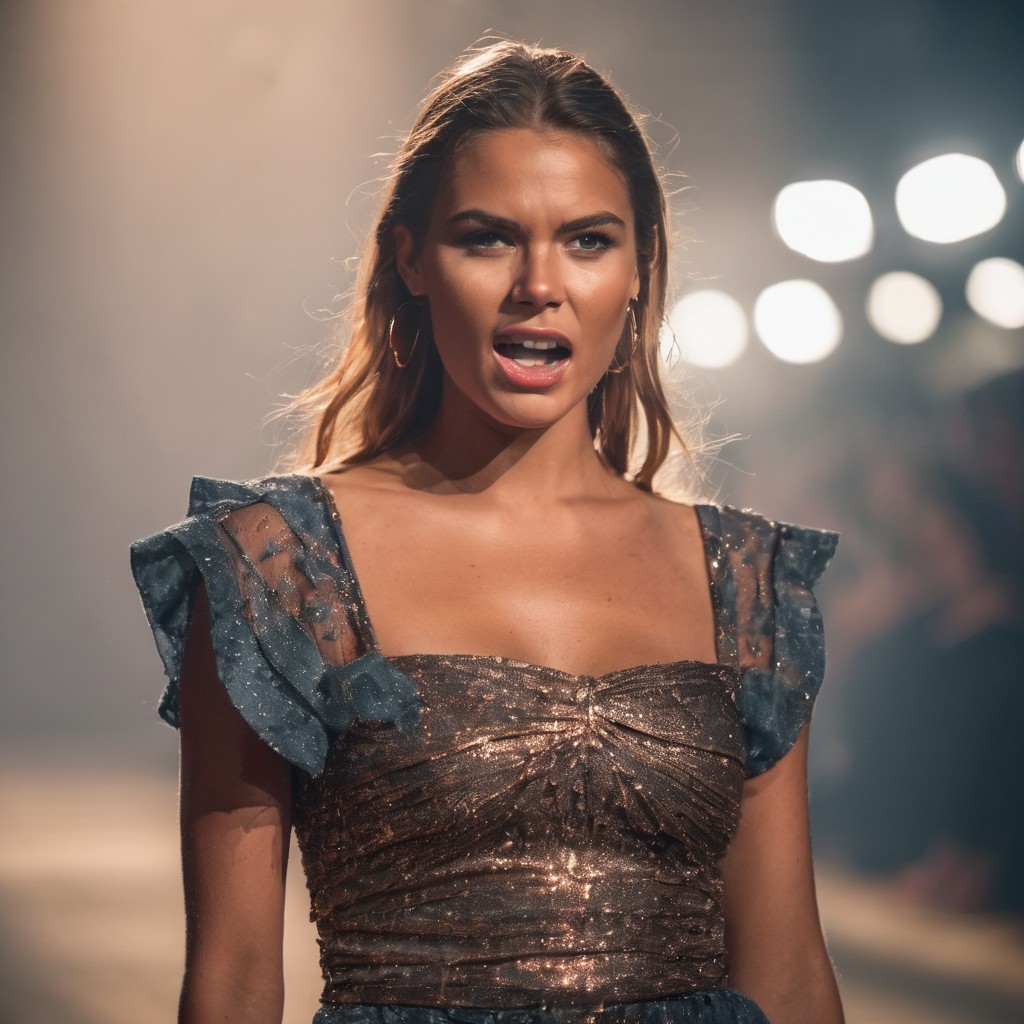}\hspace{6pt}%
\includegraphics[width=0.21\linewidth]{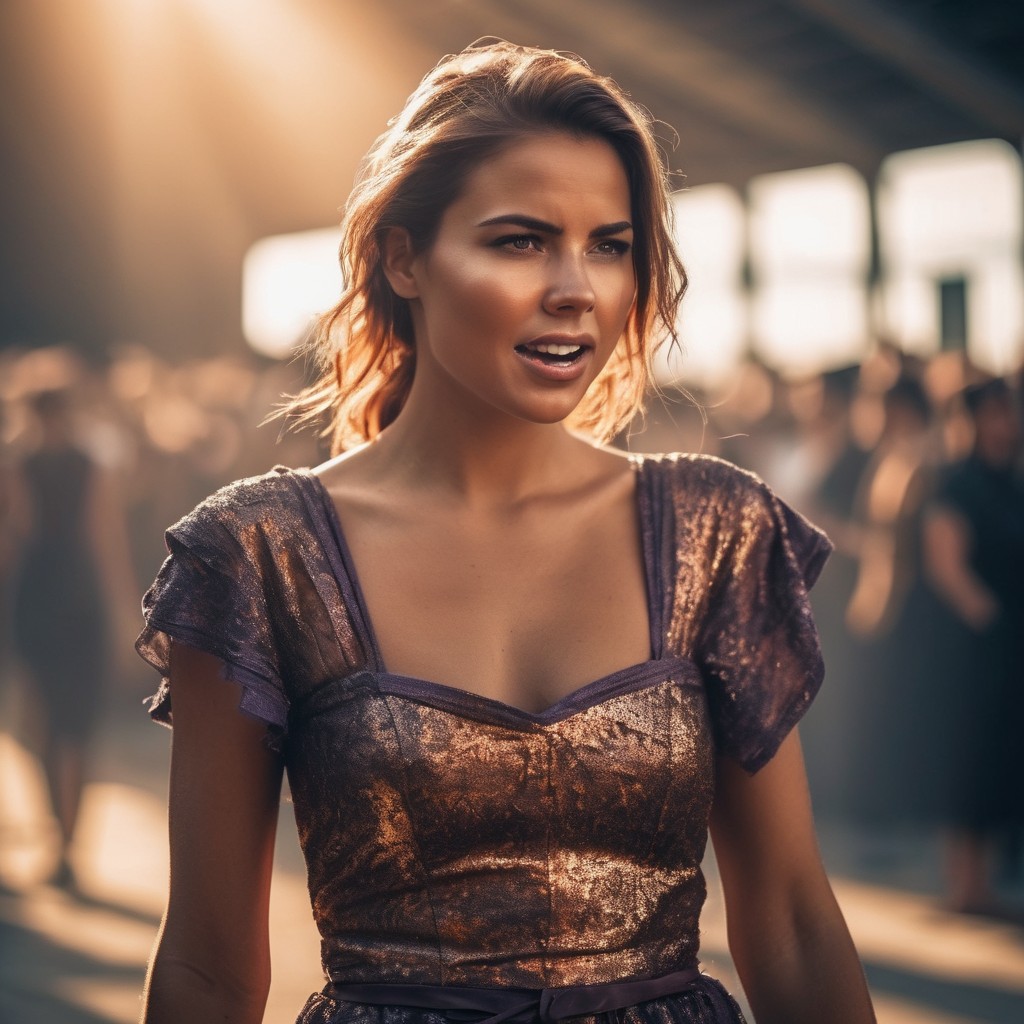}
\caption{\textbf{Pool C $N{=}5$ example.} VioletXL + catwalk + god rays + old copper + anger, SDXL. \emph{L$\to$R}: \texttt{merge}, \texttt{composite}, \LoRASpace{}. Time $36/187/42$\,s.}
\label{fig:qualitative}
\end{figure}

\subsection{Energy threshold ablation}
\label{sec:experiments:energy}

Sweeping $\tau \in \{0.95, 0.97, 0.99, 0.999\}$ on Pool A grows $\bar{D}$ by $43\%$ while quality varies non-monotonically within seed-level noise ($|\Delta| < 0.05$ CLIP-T, $|\Delta| < 0.008$ ImageReward). We adopt $\tau = 0.99$ as the default (Tab.~\ref{tab:abl_tau_poola}).

\subsection{Routing ablation: substrate is router-tolerant}
\label{sec:experiments:routing}

On Pool A across three seeds, the four routers from \S\ref{sec:method:routing} (uniform, spatial, temporal, semantic) all exceed the merge baseline by large margins (Tab.~\ref{tab:routing_abl}, App.~\ref{app:routing_full}). Temporal gives the largest ImageReward gain over uniform, semantic is flat, and spatial regresses. The substrate carries the quality, with routers contributing small modulations.

\subsection{Cross-modality predictive check on the LLM regime}
\label{sec:experiments:llm}

On the near rank-orthogonal LoRA-Retriever LLaMA-2-7B benchmark ($K{=}48$, $\rho = 0.96$), the substrate alone (\texttt{loraspace-fullpool}) scores $50.49$ Macro Task vs.\ the ambient mixture's $53.55$. INSTRUCTOR-retrieval coupling (\texttt{loraspace-hybrid}) recovers $\mathbf{54.51}$ at $1.52$ vs $1.47$ samp/s, matching the framework's prediction that compression helps only with retrieval coupling (full scoreboard in App.~\ref{app:llm_full}).

\section{Conclusion}
\label{sec:conclusion}

We empirically characterise pool-wide rank overlap in real LoRA collections, finding $\rho \in [0.46, 0.62]$ across three diffusion pools and a 48-LoRA LoRA-Hub sub-pool, with LoRA-Retriever as the near-orthogonal exception ($\rho \approx 0.96$). \LoRASpace{} turns this overlap into a pool-wide substrate from one SVD, matching composite quality at merge cost in static-mode and reaching $2.3$ to $6.1\times$ speedup in dynamic-mode, mostly from avoiding per-LoRA backbone passes. A mechanism ablation shows that the lift does not depend on the pool-wide basis, making the basis valuable for amortisation and dynamic capability (limitations and broader impact in App.~\ref{app:limitations}).

\paragraph{When to use \LoRASpace{}.}
\LoRASpace{} targets deployments that repeatedly compose several LoRAs under a memory budget. With one active adapter or unconstrained memory, ambient serving is the right choice. Within that setting, $\rho$ is a cheap screen (CPU minutes, no base model). Below about $0.6$ we deploy the substrate, as on the three diffusion pools, with LoRA-Hub ($0.62$) at the boundary. Above about $0.9$, as on LoRA-Retriever, the substrate is worth using for routing or retrieval coupling rather than for compression at the default $\tau$. In between, a $\tau$ sweep on held-out prompts decides (CLIP-T drop under $0.5$ and structural similarity at least $0.95$). These cut points separate our pools empirically and are not universal constants, and $\rho$ falls with pool size ($0.95$ at $K{=}2$ to $0.57$ at $K{=}11$, App.~\ref{app:rank_ablation}), so very small pools are better served ambiently. Uncurated and mixed pools follow the same rule (App.~\ref{app:heterogeneous}).

\begin{ack}
The computational resources for this research were supported by the Accelerating Higher Education Expansion and Development (AHEAD) Operation Grant No. 6026-LK/8743-LK from the Ministry of Higher Education, Sri Lanka, funded by the World Bank. R. Rodrigo acknowledges the University of Moratuwa Senate Research Committee Grant No. SRC/LT/2021/20. We also thank Jaseci Labs for its support.
The authors declare no competing interests.
\end{ack}

\bibliographystyle{plainnat}
\bibliography{references}

\appendix

\section{Pool-wide precompute illustration}
\label{app:precompute_fig}

\begin{figure}[h]
\centering
\includegraphics[width=0.85\linewidth]{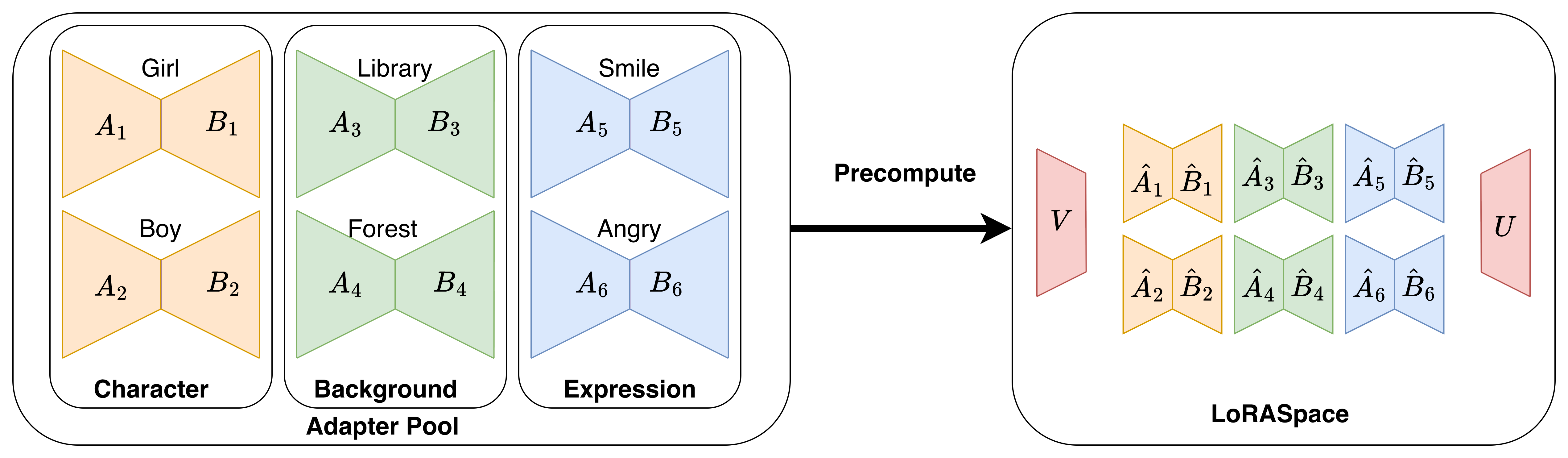}
\caption{Pool-wide precompute (paid once at pool-creation time). Per-LoRA factors are stacked and decomposed via paired marginal SVDs into shared bases $(U, V)$ (red) plus projected factors $(\hat B_i, \hat A_i)$. The cache is reused across all subsequent compositions.}
\label{fig:precompute}
\end{figure}

\section{Proof of Proposition~\ref{prop:exact_recovery}}
\label{app:proof_prop1}

\begin{proof}
Under the span conditions, $UU^\top B_i = B_i$ (because $U U^\top$ is the identity on $\mathrm{col}(U) \supseteq \mathrm{col}(B_i)$) and $A_i V V^\top = A_i$ (analogously).
Substituting Eq.~\eqref{eq:projected_factors} into Eq.~\eqref{eq:loraspace_operator}:
\begin{align*}
\Delta y^{\mathrm{LS}}(x)
  &= U \sum_i w_i \bigl(\sqrt{s_i}\, U^\top B_i\bigr)\bigl(\sqrt{s_i}\, A_i V\bigr) V^\top x \\
  &= \sum_i w_i s_i (UU^\top B_i)(A_i VV^\top) x
   = \sum_i w_i s_i B_i A_i x = \Delta y^{\star}(x). \qedhere
\end{align*}
\end{proof}

\begin{proof}[Proof of Corollary~\ref{cor:truncation_bound}]
Let $P = UU^\top$ and $Q = VV^\top$ be the projectors onto $\mathrm{col}(U)$ and $\mathrm{col}(V)$.
For each LoRA $i$, decompose $B_i = PB_i + E^B_i$ with $E^B_i = (I-P)B_i$, and $A_i = A_iQ + E^A_i$ with $E^A_i = A_i(I-Q)$.
Then
\begin{align*}
\Delta y^{\mathrm{LS}}(x) - \Delta y^{\star}(x)
  &= \sum_i w_i s_i \bigl[(PB_i)(A_iQ) - B_iA_i\bigr]x \\
  &= \sum_i w_i s_i \bigl[-B_i E^A_i - E^B_i A_i + E^B_i E^A_i\bigr]x.
\end{align*}
By the Eckart-Young theorem~\citep{eckart1936}, $\|(I-P)\bar{B}_{\mathrm{stack}}\|_2 = \sigma^B_{D+1}$.
Since $\bar{B}_i = B_i/\|B_i\|_F$ is a column-block of $\bar{B}_{\mathrm{stack}}$, we have $\|E^B_i\|_2 \le \|B_i\|_F \cdot \sigma^B_{D+1}$, and analogously $\|E^A_i\|_2 \le \|A_i\|_F \cdot \sigma^A_{D+1}$.
Applying $\|MNx\|_2 \le \|M\|_2\|N\|_2\|x\|_2$, bounding $\|A_i\|_2 \le \|A_i\|_F \le \mu^A$ and $\|B_i\|_2 \le \|B_i\|_F \le \mu^B$, summing the three error terms, and collecting over $i$ yields Eq.~\eqref{eq:truncation_bound}.
\end{proof}

\paragraph{Empirical anchor (SDXL Pool C, $\tau=0.99$).}
At $\tau=0.99$ with $K=17$ pool LoRAs, the energy-fraction definition gives $\sigma^B_{D+1} \le \sqrt{(1-\tau)K} \approx 0.41$. The corollary applies a Frobenius-norm chain on $\|A_i\|, \|B_i\|$ that loses orders of magnitude in tightness; we therefore treat Cor.~\ref{cor:truncation_bound} as a worst-case sanity statement rather than a predictive bound. Empirical robustness of $\tau=0.99$ is established by the energy-threshold sweep in \S\ref{sec:experiments:energy}, where quality varies non-monotonically within seed-level noise across $\tau \in [0.95, 0.999]$.

\section{Router specifications}
\label{app:router_specs}

The four router variants summarised in \S\ref{sec:method:routing} are detailed below. All share the same operator interface: a function producing per-LoRA weights $\{w_i\}$ that plug into Eq.~\eqref{eq:loraspace_operator}.

\paragraph{Spatial router (per-token activation energy).} Per-position normalised energy $e_i(p) = \|y_i(p)\|_2^2 / (\|x_{\mathrm{proj}}(p)\|_2^2 + \epsilon)$ scored across active LoRAs, top-$k$ selection, temperature softmax. All runs use $k{=}1$, so each token is routed to its highest-energy LoRA. Spirit of SpectR~\citep{fleshman2025spectr} but applied over per-LoRA contributions inside the pool subspace.

\paragraph{Temporal router (per-step subspace-band).} Following CMLoRA's dominant/non-dominant scheme~\citep{zou2025cmlora} but with the routing signal computed as the late-band energy fraction inside the shared subspace, $m = \sum_{d \ge D/2}\|x_{\mathrm{proj}}[\cdot, \cdot, d]\|^2 / \sum_d \|x_{\mathrm{proj}}[\cdot, \cdot, d]\|^2$. Signal and operator share the same coordinate system.

\paragraph{Semantic router (per-concept spatial mask).} A short warmup pass collects cross-attention maps for each LoRA's trigger tokens, following the token-to-region routing motivation of FreeFuse~\citep{freefuse2025}. Resulting per-LoRA masks $M_i \in [0,1]^{H \times W}$ are bilinearly resampled and applied at image-token axes, with activation-energy fallback at text-token axes.

\paragraph{External retrieval coupling.} For LLM pools, \LoRASpace{} pairs with an INSTRUCTOR retriever~\citep{su2023instructor} following the LoRA-Retriever protocol~\citep{zhao2024loraretriever}. The retrieved top-$K$ subset is composed through the same shared-subspace operator. This demonstrates that internal (activation- or band-derived) and external (retrieval-derived) routing policies share a single composition substrate.

\section{Per-method complexity comparison}
\label{app:complexity}

Table~\ref{tab:complexity_per_method} consolidates per-step compute and peak GPU memory across all evaluated methods.
The key comparison is the \emph{Scaling in $N$} column: output-space methods (\texttt{composite}, \texttt{cmlora}) scale linearly in compute and memory; PEFT-resident single-active methods (\texttt{switch}) hold per-step compute roughly constant but materialize all $K_{\mathrm{pool}}$ LoRAs on GPU regardless of how many are active per step; static-merge methods (\texttt{merge}, \texttt{knots}) achieve $O(1)$ inference cost but lose routing capability; \LoRASpace{} scales with the subspace dimension $D \le \min(\sum_i r_i, d)$, which saturates at the layer width when the pool grows large enough.

\begin{table}[h]
\centering
\small
\caption{Per-step compute and peak GPU memory of training-free multi-LoRA composition methods. $N$ is active LoRAs per call, $\bar{r}$ is representative rank, $d$ is layer width, $L$ is LoRA-bearing layers, $S$ is tokens per layer, and $K_{\mathrm{pool}}$ is pool size. \LoRASpace{}'s $D \le \min(\sum_i r_i, d)$ saturates at the layer width.}
\label{tab:complexity_per_method}
\resizebox{\textwidth}{!}{%
\begin{tabular}{lccc}
\toprule
Method & Per-step compute & Peak GPU memory & Scaling in $N$ \\
\midrule
\texttt{merge}                       & $O(L d^2)$                        & $O(L d^2)$                    & $\Theta(1)$ \\
\texttt{switch}                      & $O(L (d^2 + \bar{r} d))$          & $O(L(d^2 + K_{\mathrm{pool}} \bar{r} d))$ & compute $\Theta(1)$, mem $\Theta(K_{\mathrm{pool}})$ \\
\texttt{composite}                   & $O(N L d^2 S)$                    & $O(L(d^2 + K_{\mathrm{pool}} \bar{r} d))$ & both $\Theta(N)$ \\
\texttt{cmlora}                      & $O(N L d^2 S)$                    & $O(L(d^2 + K_{\mathrm{pool}} \bar{r} d))$ & both $\Theta(N)$ \\
\texttt{knots}                       & $O(L d^2)$                        & $O(L d^2)$                    & $\Theta(1)$ \\
\midrule
\textbf{\LoRASpace{}}
  & $O(L (d D + N \bar{r} D) S)$, $D \le \min(\textstyle\sum_i r_i, d)$
  & $O(L(d^2 + dD))$ GPU $+$ $O(L K_{\mathrm{pool}} D \bar{r})$ CPU
  & linear in $N$ inside subspace; $D$ caps with $K$ \\
\bottomrule
\end{tabular}}
\end{table}

\paragraph{Saturation thresholds.}
Saturation ($D = d$) requires $\sum_i r_i \ge d$, where $d = \min(\din, \dout)$ is the width of each layer.
SD\,1.5 LoRA layers have widths from $320$ to $1280$, and SDXL layers $640$ or $1280$.
Counted per layer, the stacked ranks of Pool C reach the layer width on all $722$ layers, and those of Pools A and B on $132$ of $192$ layers (every layer narrower than $1280$).
The diffusion pools therefore sit in the width-saturated regime on most layers, consistent with $D$ matching the spectrum-matched null on the SDXL pools (App.~\ref{app:rho_decomposition}).
The LLM pool has $\sum r_i = 384$ against $d = 4096$ ($9\%$ of threshold), deep in the rank-bound regime.

\section{Rank overlap ablation: $D$ vs.\ pool size and energy threshold}
\label{app:rank_ablation}

This appendix documents the per-pool subspace dimension ablation that underlies Table~\ref{tab:rank_fill}.

\paragraph{Energy threshold sweep (Pool A, SD\,1.5, $K=11$).}

\begin{table}[h]
\centering
\small
\caption{Pool A energy-threshold ablation. Subspace dimension $D$ and uniform-\LoRASpace{} quality at four $\tau$ values (mean over $N \in \{2, \dots, 5\}$, $240$ combos, seed $111$). Quality varies non-monotonically within seed-level noise while $D$ grows by $1.43\times$.}
\label{tab:abl_tau_poola}
\begin{tabular}{rrrrrrr}
\toprule
$\tau$ & $D_{\min}$ & $D$ (mean) & $D_{\max}$ & CLIP-T & ImageReward \\
\midrule
$0.95$  & 224 & 356 & 568 & 35.76 & 1.071 \\
$0.97$  & 248 & 392 & 611 & 35.81 & 1.072 \\
$0.99$  & 283 & 446 & 663 & 35.77 & 1.064 \\
$0.999$ & 314 & 508 & 741 & 35.80 & 1.067 \\
\bottomrule
\end{tabular}
\end{table}

\paragraph{Pool size sweep (Pool A, $\tau=0.99$).}

\begin{table}[h]
\centering
\small
\caption{Pool A rank-fill ratio $\rho$ vs.\ pool size $K$ (mean $\pm$ std over $3$ random subsets per $K$; $K{=}11$ has zero std as the full pool admits one subset). Pairs ($K{=}2$) are nearly rank-orthogonal ($\rho \approx 0.95$); at full pool, $43\%$ of rank is redundant.}
\label{tab:abl_kpool_poola}
\begin{tabular}{rrrr}
\toprule
$K$ & $\sum_i r_i$ (mean) & $D$ (mean $\pm$ std) & $\rho$ \\
\midrule
 2  & 224 & $213 \pm 50$ & 0.95 \\
 4  & 304 & $275 \pm 22$ & 0.91 \\
 6  & 456 & $353 \pm 35$ & 0.77 \\
 8  & 584 & $397 \pm 38$ & 0.68 \\
11  & 784 & $446 \pm 0$  & 0.57 \\
\bottomrule
\end{tabular}
\end{table}

\paragraph{SDXL Pool C vs.\ LLM comparison.}
Pool C ($K=17$, SDXL) reaches $\rho = 0.46$ at full pool size, with the rank-fill already at $0.69$ to $0.72$ for $K{=}4$ subsets.
This heavier overlap relative to Pool A is consistent with Pool C's higher proportion of same-type LoRAs (lighting, material, background) that share scene-element structure.
The LLM pool ($K=48$, LLaMA-2-7B~\citep{touvron2023llama2}) maintains $\rho \ge 0.96$ at all tested pool sizes with growth exponent $\beta = 0.94$, which is effectively linear and confirms near-zero rank overlap across disjoint NLP-task LoRAs.

\section{Rank-fill decomposition and matched null}
\label{app:rho_decomposition}

The rank-fill ratio $\rho = D / \sum_i r_i$ has nominal ranks in the denominator, so it mixes each adapter's internal rank slack with structure between adapters. With per-adapter effective ranks $r_i^{\mathrm{eff}}$ computed at the same $\tau$ on the same Frobenius-normalised factor stacks as $D$,
\begin{equation}
\rho = \underbrace{\frac{\sum_i r_i^{\mathrm{eff}}}{\sum_i r_i}}_{\rho_{\mathrm{intra}}} \times \underbrace{\frac{D}{\sum_i r_i^{\mathrm{eff}}}}_{\rho_{\mathrm{cross}}},
\end{equation}
where $\rho_{\mathrm{intra}}$ isolates within-LoRA redundancy. Both cross-LoRA structure and stacked ranks approaching the layer width lower $\rho_{\mathrm{cross}}$, so we also compare $D$ with a spectrum-matched random null $D_{\mathrm{null}}$, built from random per-adapter subspaces that keep each block's own singular spectrum, with the same layer shapes and widths (one seeded draw per layer, averaged over all layers). Tab.~\ref{tab:rho_decomposition} reports every pool at $\tau = 0.99$.

\begin{table}[H]
\centering
\small
\caption{Rank-fill decomposition at $\tau = 0.99$. $D / D_{\mathrm{null}} < 1$ means the pool needs fewer dimensions than a spectrum-matched random pool. The LoRA-Hub row uses a different 48 of the 196 adapters than Tab.~\ref{tab:rank_fill}, and LoRA-Retriever is an independent re-download. The stylized, community and wild pools are described in App.~\ref{app:heterogeneous}.}
\label{tab:rho_decomposition}
\begin{tabular}{lcccc}
\toprule
Pool & $\rho$ & $\rho_{\mathrm{intra}}$ & $\rho_{\mathrm{cross}}$ & $D / D_{\mathrm{null}}$ \\
\midrule
Pool C (SDXL, $K{=}17$) & 0.46 & 0.84 & 0.55 & 0.98 \\
Six stylized SDXL LoRAs ($K{=}6$) & 0.97 & 1.00 & 0.97 & 1.00 \\
Community SDXL, uncurated ($K{=}20$) & 0.87 & 0.99 & 0.88 & 1.00 \\
Pool C + community union ($K{=}37$) & 0.43 & 0.89 & 0.48 & 0.99 \\
\midrule
LoRA-Hub (FLAN-T5, $K{=}48$) & 0.60 & 1.00 & 0.60 & 0.70 \\
Wild FLAN-T5 LoRAs, uncurated ($K{=}12$) & 0.84 & 1.00 & 0.85 & 0.92 \\
LoRA-Hub + wild union ($K{=}60$) & 0.52 & 1.00 & 0.52 & 0.67 \\
LoRA-Retriever (LLaMA-2, $K{=}48$) & 0.98 & 1.00 & 0.98 & 1.00 \\
\bottomrule
\end{tabular}
\end{table}

Within-LoRA slack is material only on the two pools that contain the Pool C adapters ($\rho_{\mathrm{intra}} = 0.84$ and $0.89$) and is $0.99$ to $1.00$ elsewhere. On the LLM pools, far from width saturation, the related-task pools need about $30\%$ fewer dimensions than their matched nulls ($0.70$ for LoRA-Hub and $0.67$ for its union with an uncurated harvest), while LoRA-Retriever sits at its null. This alignment is consistent with shared task structure, but the control does not isolate training data as the cause. On the SDXL pools, whose stacked ranks exceed the layer widths, $D$ sits at the null ($0.98$ to $1.00$). Their low $\rho$ is real compressibility, largely explained by width saturation and per-adapter spectra rather than by measured alignment, and it does not depend on a curated or related pool. Repeating the decomposition at the compressed operating points of App.~\ref{app:joint_basis} instead of $\tau = 0.99$ attributes $75$ to $83\%$ of the log-compression to within-adapter tails on both pool types tested.

\section{Joint versus marginal bases and compressed thresholds}
\label{app:joint_basis}

\paragraph{Joint versus marginal SVD.}
We compare the paired marginal bases of \S\ref{sec:method:basis} with bases from a joint SVD of the stacked scaled deltas $s_i B_i A_i$, at the same rank $D$ (the $D$ the marginal method selects at $\tau = 0.99$), so capacity is equal. Tab.~\ref{tab:joint_vs_marginal} reports relative Frobenius reconstruction error of the per-LoRA deltas and of composed uniform $N{=}5$ deltas on Pool C, on $24$ layers spanning nine layer types across the down, mid and up blocks, then on all $722$ layers.

\begin{table}[H]
\centering
\small
\caption{Reconstruction error of marginal and joint bases at equal $D$ (Pool C).}
\label{tab:joint_vs_marginal}
\begin{tabular}{lccccc}
\toprule
 & & \multicolumn{2}{c}{Per-LoRA error} & \multicolumn{2}{c}{Composed error ($N{=}5$)} \\
\cmidrule(lr){3-4}\cmidrule(lr){5-6}
Setting & $D$ & marginal & joint & marginal & joint \\
\midrule
Full pool, $K{=}17$, 24 layers & 822 & 0.88\% & 0.35\% & 0.84\% & 0.43\% \\
Small pool, $K{=}6$, 24 layers & 374 & 0.52\% & 0.22\% & 0.56\% & 0.32\% \\
Full pool, $K{=}17$, all 722 layers & per layer & 1.01\% & 0.43\% & 0.95\% & 0.53\% \\
\bottomrule
\end{tabular}
\end{table}

The joint construction is closer to product-optimal in weight space, with worst cases of a few percent on low-energy layers ($8.7\%$ marginal against $7.3\%$ joint). The difference does not reach generation. A cache built from the joint bases at the same per-layer $D$ regenerates all $30$ benchmark combinations at DINO similarity $0.999$ to the marginal-cache images (CLIP-T $25.87$ against $25.92$). The joint spectrum is also more compact ($D \approx 148$ on these layers against $822$). A joint core $U^\top (s_i B_i A_i) V$ has rank at most $r_i$ and refactors exactly into thin factors, so at matched $D$ both caches store the same number of parameters, while at its own $\tau = 0.99$ the joint cache is about $5\times$ smaller ($1.77$ against $8.63$\,GB in fp16) and builds $14\%$ faster.

We keep the marginal bases as the default because they can be built and grown from native thin factors in a small memory budget. The marginal build streams thin factor stacks and never forms a dense $\dout \times \din$ delta, so its peak transient memory per layer is $83$\,MB against $1.3$\,GB for our joint build, which keeps the precompute on CPU ($21.8$\,s for Pool A). The marginal coordinates also support the append-only extension of App.~\ref{app:dynamic_pools} and the exact recovery of Prop.~\ref{prop:exact_recovery}. The joint substrate is the more compact choice when the dense build is affordable, and the marginal one suits pools that change.

\paragraph{Compressed thresholds.}
Tab.~\ref{tab:tau_compressed} lowers $\tau$ below the default after submission ($105$ images, $15$ combinations at $N{=}5$, an $18$-adapter manifest with mean $D = 785$ at $\tau = 0.99$, and our own prompt set, so CLIP-T is not comparable to the main tables). These compressed caches, and those of the next paragraph, are made by truncating the $\tau = 0.99$ cache to $\min(D_B, D_A)$ per layer rather than the $\max(D_B, D_A)$ of \S\ref{sec:method:basis}, so at a given $\tau$ they are no larger than a fresh build. CLIP-T and ImageReward stay flat within run-to-run variation down to $D \approx 50$, the same order as the joint spectrum's $D$.

\begin{table}[H]
\centering
\small
\caption{Quality at compressed energy thresholds. DINO is the similarity to the $\tau = 0.99$ images.}
\label{tab:tau_compressed}
\begin{tabular}{ccccc}
\toprule
$\tau$ & mean $D$ & CLIP-T & ImageReward & DINO \\
\midrule
0.99 & 785 & 25.33 & 1.137 & 1.000 \\
0.70 & 107 & 25.67 & 1.158 & 0.981 \\
0.50 & 49 & 25.61 & 1.148 & 0.967 \\
0.30 & 22 & 25.27 & 1.117 & 0.949 \\
\bottomrule
\end{tabular}
\end{table}

\paragraph{The basis matters at compressed sizes.}
At $\tau = 0.99$ the choice of basis does not change quality (\S\ref{sec:experiments:quality}). At compressed sizes it does. Regenerating $30$ fixed-seed combinations at identical per-layer $D$ with four bases at $\tau = 0.50$ ($D \approx 49$, Pool C cache $0.63$\,GB against $9.5$\,GB), the pool basis keeps DINO similarity $0.967$ to the full-cache images (CLIP-T $25.93$ against $25.92$), while a random basis and a basis built from a different pool fall to $0.877$ and $0.901$ because the adapters no longer come through. A per-combination Eckart-Young basis also holds quality ($0.986$) but costs minutes of dense SVDs per combination and leaves no reusable object. The pool basis keeps $0.956$ and $0.945$ at $\tau = 0.40$ and $0.30$ (caches of $0.40$ and $0.26$\,GB), while both controls stay at $0.88$ to $0.90$.

\section{Heterogeneous and uncurated pools}
\label{app:heterogeneous}

Most of our pools are earlier benchmarks adopted unchanged. To test a pool nobody on our team chose, we harvested the $20$ most-downloaded valid community SDXL LoRAs by a fixed rule (Hugging Face SDXL tree, downloads descending, $16$ creators, ranks $4$ to $128$) and measured them on the same $722$ layers as Pool C. We also added six strongly stylized SDXL LoRAs. Tab.~\ref{tab:heterogeneous} reports $\rho$ at $\tau = 0.99$.

\begin{table}[H]
\centering
\small
\caption{Rank fill of heterogeneous and uncurated SDXL pools at $\tau = 0.99$.}
\label{tab:heterogeneous}
\begin{tabular}{lrrrr}
\toprule
Pool & $K$ & $\sum_i r_i$ & $D$ & $\rho$ \\
\midrule
Pool C (control) & 17 & 1792 & 819 & 0.46 \\
Six stylized LoRAs alone & 6 & 192 & 186 & 0.97 \\
Pool C + six stylized & 23 & 1984 & 848 & 0.43 \\
Community LoRAs, uncurated & 20 & 673 & 580 & 0.87 \\
Pool C + community & 37 & 2337 & 921 & 0.43 \\
\bottomrule
\end{tabular}
\end{table}

Unrelated adapters among themselves stay nearly incompressible ($\rho = 0.97$ for the stylized set and $0.87$ for the community set), as on the LLM side (an uncurated harvest of $12$ FLAN-T5 task LoRAs at $0.84$, LoRA-Retriever at $0.96$). A mixed pool still deploys. Adding the six stylized adapters grows the pooled $D$ by only $3.5\%$ ($819$ to $848$, against $186$ for the stylized set alone), and in a mixed $17$-adapter pool of Pool C and stylized adapters every adapter keeps at least $99.9\%$ of its delta energy ($12$ layers). New adapters are detectable before use. Every community adapter lies outside the Pool C bases (captured energy $0.15$ to $0.64$, mean $0.35$, against $0.999$ or more for in-pool adapters), and every one is fully represented after the extension of App.~\ref{app:dynamic_pools} (captured energy $1.00$ with about $27$ extra directions per layer). The union also absorbed engineering differences, such as varying per-layer ranks, unsupported layer types and file-format differences, which the merge path had to normalise case by case.

On the LLM side, the $12$ harvested FLAN-T5 adapters (car-rental dialogue, SQuAD question answering, meeting summarisation and Open-Orca instruction following, from different creators) need $D = 155$ alone but add only $35$ dimensions to the $48$-task LoRA-Hub substrate, and the union sits further below its matched null than LoRA-Hub itself ($0.67$ against $0.70$, App.~\ref{app:rho_decomposition}).

\paragraph{End-to-end check on the union.}
One substrate over the $37$-adapter union reproduces the equal-weight ambient composition on $30$ combinations (DINO similarity $0.999$ against per-combination PEFT merge). A no-adapter base arm on the same prompts and seed scores $0.90$ and $0.85$ at $N{=}3$ and $N{=}5$ against the same merges, so the metric discriminates. This control bounds metric headroom only. The $20$ combinations at $N{=}3$ use community adapters without trigger words, so their prompts and base image are identical.

\section{Changing pools: basis extension and tolerance}
\label{app:dynamic_pools}

We measure what happens when a LoRA arrives after the substrate is built, on Pool C ($24$ layers, energy-weighted, leave-one-out over all $17$ adapters).

\paragraph{Projection onto the unchanged bases.}
A held-out adapter's delta projected onto the substrate built from the other $16$ keeps on average only $0.30$ of its energy ($0.21$ to $0.45$), and composed $N{=}5$ error averages $0.34$ when that adapter is used immediately. A larger existing pool does not help (captured energy $0.10$, $0.18$, $0.26$ and $0.30$ at $K = 4, 8, 12, 16$), and twenty uncurated community adapters capture only $0.15$ to $0.64$ on the Pool C bases (App.~\ref{app:heterogeneous}). Each adapter brings its own few effective directions, so a new LoRA lies largely outside the existing span.

\paragraph{Basis extension.}
For each layer we take the new LoRA's residual against the existing bases, $R_B = (I - UU^\top) B_{\mathrm{new}}$ and $R_A = A_{\mathrm{new}} (I - VV^\top)$, and append its top-$q$ singular directions to $(U, V)$. Existing coordinates are untouched. Admission takes $33$\,s on CPU over all $722$ layers including the cache write, against $7$\,min $41$\,s on GPU for a full rebuild at $K{=}37$. Tab.~\ref{tab:extension} reports captured energy. Reaching $0.99$ needs on average only $12$ extra directions per layer ($5$ to $28$) against $D \approx 820$.

\begin{table}[H]
\centering
\small
\caption{Captured energy of a held-out Pool C adapter under each update path (mean over all $17$ adapters).}
\label{tab:extension}
\begin{tabular}{lcc}
\toprule
Update path & Captured energy & Extra directions per layer \\
\midrule
Projection, unchanged bases & 0.30 & 0 \\
Extension, $q = 4$ & 0.92 & 4 \\
Extension, $q = 8$ & 0.96 & 8 \\
Extension, $q = 16$ & 0.98 & 16 \\
Extension, $q = 32$ & 1.00 & 32 \\
\bottomrule
\end{tabular}
\end{table}

\paragraph{End-to-end check.}
We rebuilt both caches and regenerated $30$ fixed-seed $N{=}5$ combinations. The extension cache reproduces the full-recompute images at DINO similarity $0.999$ (worst $0.995$) on the $10$ combinations that contain the new LoRA, with CLIP-T unchanged. The projection cache drifts on exactly those combinations ($0.982$, worst $0.936$) and nowhere else ($0.999$ on the other $20$).

\paragraph{Sequential arrivals.}
Starting from a $6$-adapter core and admitting the other $11$ adapters one at a time, extension keeps $D$ smaller than a fresh recompute (final mean $D$ of $480$ against $822$). Composed error grows from $3.3\%$ after the first arrival to $6.4\%$ after the eleventh, against $0.84\%$ for a fresh build. This motivates the update rule in App.~\ref{app:limitations}, which is to extend on arrival, monitor composed error, and rebuild periodically off the serving path.

\section{Capability comparison}
\label{app:capabilities}

Tab.~\ref{tab:capabilities} compares \LoRASpace{} with the closest training-free methods on four requirements of the deployment we target, many LoRAs composed under a memory budget. The last two rows are measured rather than asserted (\S\ref{sec:experiments:overlap}, App.~\ref{app:dynamic_pools}).

\begin{table}[H]
\centering
\scriptsize
\setlength{\tabcolsep}{3pt}
\caption{Requirements of multi-LoRA composition under a memory budget. Composite covers ComposLoRA~\citep{zhong2024composlora} and CMLoRA~\citep{zou2025cmlora}.}
\label{tab:capabilities}
\begin{tabular}{@{}>{\raggedright\arraybackslash}p{4.2cm}cccccc@{}}
\toprule
Requirement & CtS & KnOTS & Core-Space & Composite & Per-combo truncated merge & \LoRASpace{} \\
\midrule
Serve any active subset, static or routed, from one precompute & no & no & no & n/a & no & yes \\
Context-dependent weights without $N$ backbone passes & no & no & no & no & no & yes \\
Predict before deployment whether a pool compresses & partial$^\ast$ & no & no & no & no & yes ($\rho$) \\
Admit a newly published adapter without recomputing the pool & no & no & no & n/a & no & yes (extension) \\
\bottomrule
\end{tabular}
\\[2pt]
{\scriptsize $^\ast$For single-adapter serving only.}
\end{table}

\section{High-cardinality prompt truncation mitigation}
\label{app:truncation}

CLIP-L's 77-token hard limit is reached in all $15$ sampled Pool C combos at $N{=}7$ (maximum prompt length $92$ tokens), silently dropping the last-listed LoRA's trigger from the cross-attention context.
A deterministic $20$-line trigger-deduplication rule (\texttt{compact}) fits every combo under $74$ tokens and recovers $+0.36$ GPT-4o composition points on average across the three \LoRASpace{} variants (Table~\ref{tab:truncation}).
The mitigation is orthogonal to the composition operator and applies equally to all multi-LoRA methods.
Pool C headline numbers use native triggers (identical truncation applied to every method); Pool B high-cardinality cells ($N=7,8$) use \texttt{compact}.

\begin{table}[h]
\centering
\small
\caption{GPT-4o composition and image quality under native (truncated) triggers versus the \texttt{compact} deduplication rule, $15$ Pool C combos at $N{=}7$.}
\label{tab:truncation}
\resizebox{\textwidth}{!}{%
\begin{tabular}{lcccc}
\toprule
Method & Composition (native) & Composition (\texttt{compact}) & $\Delta$ composition & $\Delta$ image \\
\midrule
\texttt{spatial-\LoRASpace{}}  & 4.15 & 4.27 & $+0.12$ & $-0.19$ \\
\texttt{temporal-\LoRASpace{}} & 3.82 & 4.38 & $\mathbf{+0.56}$ & $-0.01$ \\
\texttt{semantic-\LoRASpace{}} & 3.78 & 4.20 & $+0.42$ & $-0.13$ \\
\midrule
Mean & 3.92 & 4.28 & $+0.36$ & $-0.11$ \\
\bottomrule
\end{tabular}%
}
\end{table}

\section{Pool B (SD\,1.5, broader pool): full high-cardinality results}
\label{app:pool_b}

\begin{table}[h]
\centering
\small
\caption{Pool B (SD\,1.5, $17$ LoRAs, $48$ combos per cell) at $N \in \{7, 8\}$ with \texttt{compact} triggers. \textbf{Bold} = best; \underline{underline} = best non-\LoRASpace{}.}
\label{tab:poolb_n78}
\resizebox{\textwidth}{!}{%
\begin{tabular}{l|ccc|ccc}
\toprule
 & \multicolumn{3}{c|}{$N{=}7$} & \multicolumn{3}{c}{$N{=}8$} \\
Method & CLIP-T & ImgReward & DINO (vs.\ comp.) & CLIP-T & ImgReward & DINO (vs.\ comp.) \\
\midrule
\texttt{merge}       & 32.16 & 0.06  & 0.508          & 32.49 & 0.01  & 0.434 \\
\texttt{switch}      & \underline{\textbf{36.43}} & \underline{\textbf{1.23}} & 0.689
                     & \underline{\textbf{36.10}} & 1.10  & 0.650 \\
\texttt{composite}   & 35.65 & 0.61  & n/a            & 34.47 & 0.39  & n/a \\
\texttt{cmlora}      & 34.59 & 0.60  & \underline{\textbf{0.833}}
                     & 34.60 & 0.38  & \underline{\textbf{0.809}} \\
\midrule
\texttt{spatial-\LoRASpace{}}  & 34.79 & 0.70          & 0.512 & 34.03 & 0.59          & 0.399 \\
\texttt{temporal-\LoRASpace{}} & 34.58 & 1.07          & \textbf{0.698}
                                & 33.81 & \textbf{1.19} & \textbf{0.625} \\
\texttt{semantic-\LoRASpace{}} & 34.81 & 1.04          & 0.679 & 33.76 & 1.14          & 0.615 \\
\bottomrule
\end{tabular}%
}
\end{table}

\section{Routing ablation: full table}
\label{app:routing_full}

Tab.~\ref{tab:routing_abl} reports the full per-router quality for the routing ablation summarised in \S\ref{sec:experiments:routing}.

\begin{table}[h]
\centering
\small
\caption{Routing ablation on Pool A with the \LoRASpace{} substrate fixed (mean $\pm$ across-seed std over 3 seeds, $240$ combos/seed, $N \in \{2, \dots, 5\}$). All four routers exceed the merge baseline ($33.60$ / $0.760$, Tab.~\ref{tab:main_quality}); temporal gives the largest ImageReward gain over uniform.}
\label{tab:routing_abl}
\begin{tabular}{lrr}
\toprule
Router & CLIP-T $\uparrow$ & ImageReward $\uparrow$ \\
\midrule
uniform              & $35.82 \pm 0.10$ & $1.050 \pm 0.029$ \\
spatial              & $35.43 \pm 0.34$ & $0.857 \pm 0.062$ \\
temporal             & $\mathbf{35.92 \pm 0.10}$ & $\mathbf{1.086 \pm 0.030}$ \\
semantic             & $35.83 \pm 0.09$ & $1.050 \pm 0.026$ \\
\bottomrule
\end{tabular}
\end{table}

\section{Mechanism ablation: per-$N$ breakdown}
\label{app:mechanism}

The mechanism comparison in \S\ref{sec:experiments:quality} reports means over $N=2$ to $5$. Tab.~\ref{tab:mechanism_app} gives the per-$N$ breakdown for both rank-$D$ static-delta variants on Pool A (single seed $111$, $240$ combos total). \texttt{subspace-merge} (pool-wide $(U,V)$ baked into a static delta) and \texttt{truncated-merge} (per-combo SVD of the ambient merge) track each other within $\le 0.12$ CLIP-T and $\le 0.04$ ImageReward at every $N$, supporting the body conclusion that the basis choice does not affect quality at equal rank $D$. Both also track per-step uniform \LoRASpace{} (Tab.~\ref{tab:main_quality}) within fp16 noise.

\begin{table}[h]
\centering
\small
\caption{Mechanism ablation per $N$. \texttt{truncated-merge} (combo-specific basis) vs.\ \texttt{subspace-merge} (pool-wide basis), Pool A, seed $111$. Both run at ambient-merge inference cost; they differ only in configure-time precompute.}
\label{tab:mechanism_app}
\begin{tabular}{lrr|rr}
\toprule
 & \multicolumn{2}{c|}{\texttt{truncated-merge}} & \multicolumn{2}{c}{\texttt{subspace-merge}} \\
$N$ & CLIP-T $\uparrow$ & ImgRew $\uparrow$ & CLIP-T $\uparrow$ & ImgRew $\uparrow$ \\
\midrule
2 & 34.57 & 1.053 & 34.45 & 1.022 \\
3 & 35.31 & 1.037 & 35.39 & 1.028 \\
4 & 36.13 & 1.078 & 36.11 & 1.081 \\
5 & 37.00 & 1.126 & 37.10 & 1.115 \\
\midrule
Mean & 35.75 & 1.073 & 35.76 & 1.062 \\
\bottomrule
\end{tabular}
\end{table}

\section{Pool C quality across active-LoRA sizes}
\label{app:pool_c}

Pool C uses $17$ SDXL LoRAs across six categories (character, background, lighting, material, expression, accessory). Quality is measured at $N\in\{2,3,4,5,6\}$ across $40$ stratified combinations per $N$, seed $111$, $50$ DPM-Solver++ steps, $1024\times1024$.

\begin{table}[h]
\centering
\small
\caption{Pool C (SDXL, $17$ LoRAs) quality per $N$. Seed $111$, $200$ combos per method ($40$ per $N$, identical across methods). $95\%$ bootstrap CIs (mean over $N$): merge CLIP-T $[28.97, 29.86]$, ImgRew $[0.60, 0.84]$; composite $[31.5, 32.4]$, $[0.88, 1.06]$; uniform \LoRASpace{} $[32.5, 33.3]$, $[0.95, 1.13]$. Merge ImgRew drops from $1.30$ at $N{=}2$ to $0.18$ at $N{=}6$; \LoRASpace{} retains $0.69$.}
\label{tab:pool_c_quality}
\begin{tabular}{lrr|rr|rr}
\toprule
 & \multicolumn{2}{c|}{merge} & \multicolumn{2}{c|}{composite} & \multicolumn{2}{c}{\LoRASpace{} (uniform)} \\
$N$ & CLIP-T $\uparrow$ & ImgRew $\uparrow$ & CLIP-T $\uparrow$ & ImgRew $\uparrow$ & CLIP-T $\uparrow$ & ImgRew $\uparrow$ \\
\midrule
2 & 29.33 & 1.300 & 31.50 & 1.292 & 31.67 & 1.289 \\
3 & 31.32 & 0.894 & 32.95 & 1.108 & 33.60 & 1.089 \\
4 & 28.82 & 0.799 & 31.79 & 1.058 & 32.95 & 1.106 \\
5 & 28.53 & 0.426 & 30.48 & 0.829 & 32.01 & 1.028 \\
6 & 29.06 & 0.184 & 32.85 & 0.577 & 34.22 & 0.691 \\
\midrule
Mean & 29.41 & 0.721 & 31.91 & 0.973 & \textbf{32.89} & \textbf{1.040} \\
\bottomrule
\end{tabular}
\end{table}

\section{Speed and precompute cost}
\label{app:pool_c_speed}

Tables~\ref{tab:speed_memory_full} and~\ref{tab:pool_a_speed} report per-image wall-clock and GPU memory across all evaluated methods on Pool C (SDXL) and Pool A (SD\,1.5) respectively. Precompute cost is discussed below.

\begin{table}[H]
\centering
\small
\caption{Per-image speed and GPU memory on Pool C (SDXL, $1024^2$, $50$ steps, RTX\,3090). Peak VRAM measured at $N{\in}\{2,4,6\}$. $N{=}7$ entries linearly extrapolated. Static-mode \LoRASpace{} (\texttt{subspace-merge}) inherits ambient \texttt{merge}'s per-image cost by construction.}
\label{tab:speed_memory_full}
\resizebox{\textwidth}{!}{%
\begin{tabular}{lcccc}
\toprule
Method & Time @ $N{=}2$ (s) & Time @ $N{=}7$ (s) & Peak VRAM & Load VRAM \\
\midrule
\texttt{merge}     & \textbf{23.0} & \textbf{43.5}$^\dagger$ & 22.34\,GB & \textbf{18.19\,GB} \\
\texttt{switch}    & 23.3 & 23.4  & 22.83\,GB & 19.68\,GB \\
\texttt{composite} & 76.3 & 262.2 & 22.83\,GB & 19.69\,GB \\
\texttt{cmlora}    & 77.1 & 261.5 & 22.83\,GB & 19.69\,GB \\
\midrule
\texttt{spatial-\LoRASpace{}}  & 34.6 & 48.0 & 15.60 to 16.80\,GB & 6.46\,GB \\
\texttt{temporal-\LoRASpace{}} & 32.9 & 42.7 & 15.60 to 16.80\,GB & 6.46\,GB \\
\texttt{semantic-\LoRASpace{}} & 32.9 & 42.5 & 15.60 to 16.80\,GB & 6.46\,GB \\
\bottomrule
\end{tabular}%
}\\[2pt]
{\footnotesize $^\dagger$\,Extrapolated from measured $N{=}2{:}23.0$, $N{=}4{:}31.7$, $N{=}6{:}39.4$ (linear fit, slope $4.1$\,s/LoRA).}
\end{table}

\begin{table}[H]
\centering
\small
\caption{Per-image speed and GPU memory on Pool A (SD\,1.5, $512^2$, $200$ steps, RTX\,3090, $10$ combos, $2$ warmup $+ 3$ timed repeats per cell, mean across repeats). Body uses $100$ steps so wall-clock halves at deployment; relative ratios are preserved. \texttt{spatial-LoRASpace{}} cells contain measurement noise from the $\le 3$ repeats; the dominant signal is the $N$-scaling separation between \LoRASpace{} and \texttt{composite}.}
\label{tab:pool_a_speed}
\resizebox{\textwidth}{!}{%
\begin{tabular}{lcccc}
\toprule
Method & Time @ $N{=}2$ (s) & Time @ $N{=}5$ (s) & Peak VRAM & Load VRAM \\
\midrule
\texttt{merge}     & \textbf{18.9} & \textbf{32.7} & 3.49\,GB & \textbf{2.88\,GB} \\
\texttt{switch}    & 17.5          & 22.0          & 3.49\,GB & 2.88\,GB \\
\texttt{composite} & 68.2          & 142.6         & 3.49\,GB & 2.89\,GB \\
\texttt{cmlora}    & 59.5          & 138.5         & 3.49\,GB & 2.89\,GB \\
\midrule
\texttt{spatial-\LoRASpace{}}  & 37.6 & 30.7 & 4.09\,GB & 2.88\,GB \\
\texttt{temporal-\LoRASpace{}} & 28.7 & 44.3 & 4.09\,GB & 2.88\,GB \\
\texttt{semantic-\LoRASpace{}} & 31.7 & 34.0 & 4.09\,GB & 2.88\,GB \\
\bottomrule
\end{tabular}%
}
\end{table}

\paragraph{Precompute cost.}
Precompute is a one-time pool-level cost: load every LoRA's factor pair, build the energy-normalised $B$- and $A$-stacks per layer, run two thin SVDs, and project per-LoRA factors into the truncated bases.
For Pool A (SD\,1.5, $K{=}11$, $\sum_i r_i = 784$, $192$ LoRA-bearing layers) the dominant per-layer SVD is on a $\dout \times \sum r_i$ matrix with $\dout \in \{320, 640, 1280\}$; the streaming script runs end-to-end in $\mathbf{21.8}$\,s SVD time + $0.7$\,s pickle write on CPU (single-thread Intel Xeon) and produces an $866$\,MB cache.
For Pool C (SDXL, $K{=}17$, $\sum_i r_i = 1792$, $722$ LoRA-bearing layers) the same pipeline produces a $9.54$\,GB cache; the lazy per-layer streaming precompute peaks at $10.1$\,GiB of host memory (maximum resident set size).
The cache is loaded once at pipeline initialisation (Pool A: $0.6$\,s; Pool C: $5.4$\,s on a SATA SSD) and is shared across all subsequent generations: precompute is amortised across the entire experiment and never enters the per-image latency budget.

\paragraph{Where the dynamic-mode speedup comes from.}
To separate backbone-pass avoidance from the subspace path, we add a single-pass ambient mixture that keeps all $N$ adapters active as PEFT modules and runs one backbone pass (Tab.~\ref{tab:latency_split}, measured on Pool C with $15$ combinations at $N{=}5$, seed $111$, $50$ steps, $1024^2$, idle GPU, medians over the $14$ inter-image intervals of one sequential run per condition). Of the $5.3\times$ gain of \LoRASpace{} over composite at $N{=}5$, $4.9\times$ comes from running one backbone pass instead of five and $1.08\times$ from the substrate path. In isolation, the subspace path at $\tau = 0.99$ is $6$ to $7\%$ slower than the ambient factor path ($47.6$ against $44.4$\,ms at $4096$ tokens and $29.9$ against $28.0$\,ms at $1024$ tokens, per-layer microbenchmark summed over $48$ sampled layers with the real cache factors), as Eq.~\eqref{eq:complexity} predicts when $D > N \bar{r}$. The $1.08\times$ therefore reflects the overhead of keeping $N$ PEFT adapters active rather than an arithmetic advantage of the subspace. The ambient condition supports per-step weight changes, but per-token routing would need further implementation.

\begin{table}[H]
\centering
\small
\caption{Pool C per-image time at $N{=}5$, with a single-pass ambient mixture between composite and \LoRASpace{}.}
\label{tab:latency_split}
\begin{tabular}{lccr}
\toprule
Condition & Backbone passes & Adapter path & Median s/image \\
\midrule
\texttt{composite} & $5$ & none (per-pass guidance) & $181.9$ \\
Ambient single-pass & $1$ & $N$ thin factor pairs & $37.3$ \\
\LoRASpace{} & $1$ & shared subspace (Eq.~\eqref{eq:loraspace_operator}) & $34.5$ \\
\bottomrule
\end{tabular}
\end{table}

\clearpage

\section{Additional qualitative examples}
\label{app:qualitative_extra}

This appendix supplements Fig.~\ref{fig:qualitative} with paired generations for further Pool C combos. Within each row the prompt and seed are held fixed across all three methods, so visual differences reflect the composition operator only. We do not show Pool A or Pool B generations, because their character LoRAs depict real people (App.~\ref{app:limitations}). For reference, Pool A per-image cost on RTX\,3090 (SD\,1.5, $512^2$, 100 DPM-Solver++ steps) is $\sim$16\,s (merge), $\sim$58\,s at $N{=}4$ or $\sim$71\,s at $N{=}5$ (composite), and $\sim$17\,s at $N{=}4$ or $\sim$15\,s at $N{=}5$ (\LoRASpace{}), with peak VRAM $3.5$\,GB for merge and composite and $4.1$\,GB for \LoRASpace{}.

\paragraph{Pool C qualitative examples (SDXL, $1024^2$, 50 DPM-Solver++ steps).} Per-image cost on RTX\,3090: merge $\sim$31\,s ($N{=}4$) / $\sim$36\,s ($N{=}5$) / $\sim$39\,s ($N{=}6$); composite $\sim$151\,s ($N{=}4$) / $\sim$187\,s ($N{=}5$) / $\sim$225\,s ($N{=}6$); \LoRASpace{} $\sim$40\,s ($N{=}4$) / $\sim$42\,s ($N{=}5$) / $\sim$45\,s ($N{=}6$). Peak VRAM: merge $22.3$\,GB, composite $22.8$\,GB, \LoRASpace{} $15.6$ to $16.8$\,GB.

\begin{figure}[!htbp]
\centering
\includegraphics[width=0.31\linewidth]{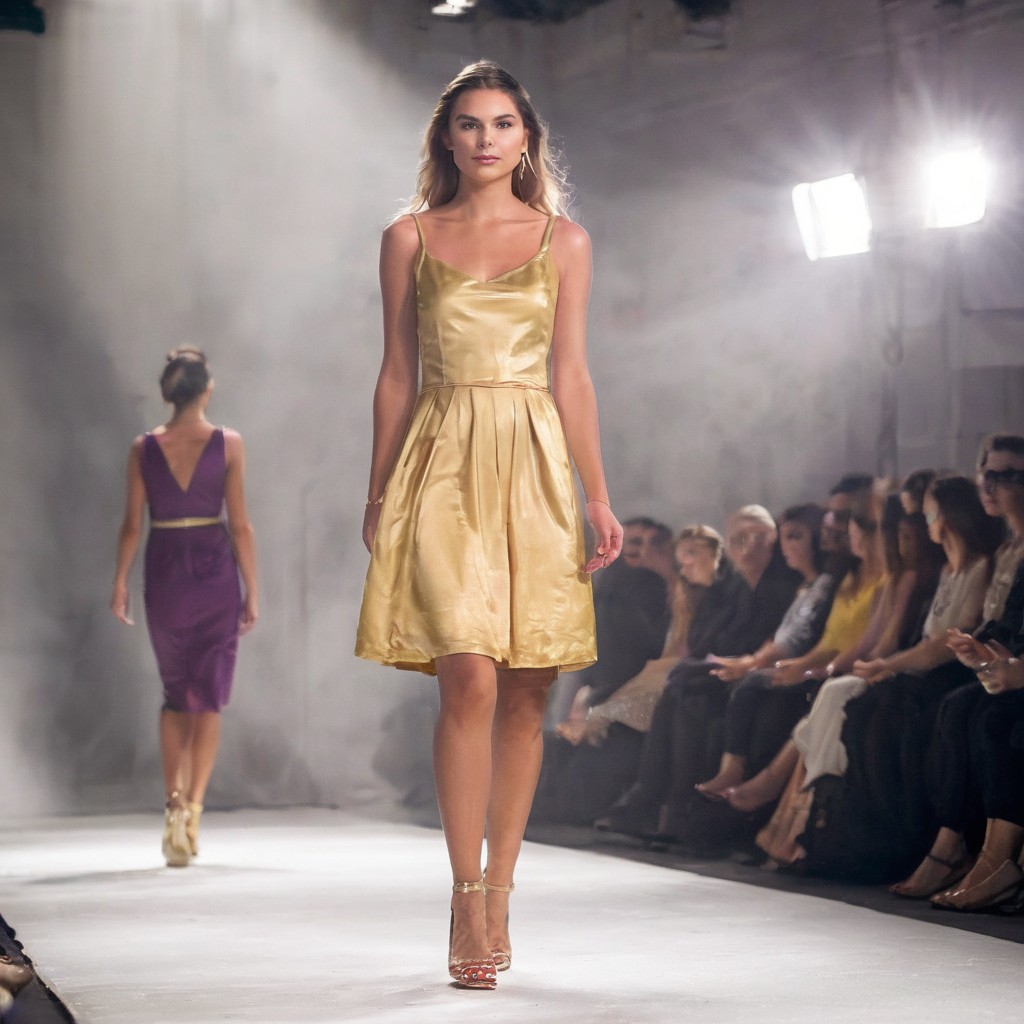}\hspace{6pt}%
\includegraphics[width=0.31\linewidth]{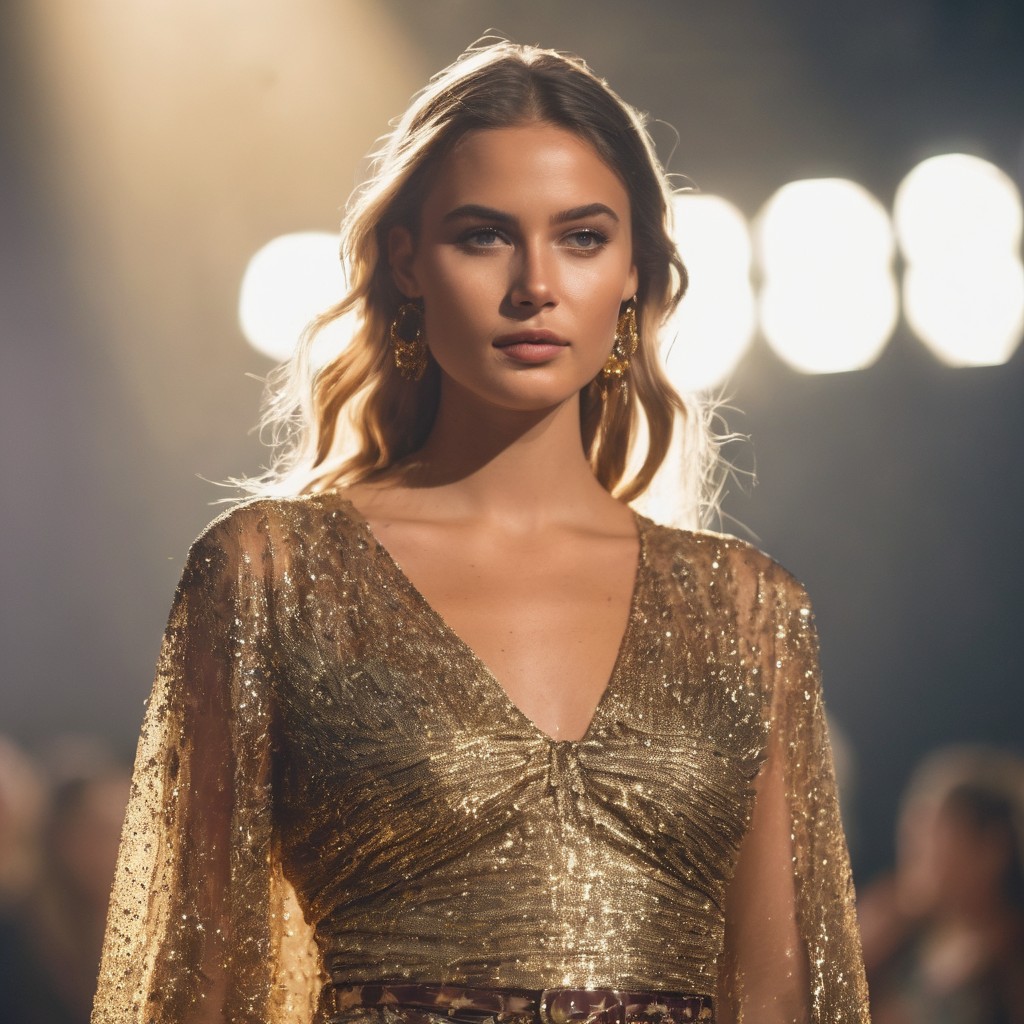}\hspace{6pt}%
\includegraphics[width=0.31\linewidth]{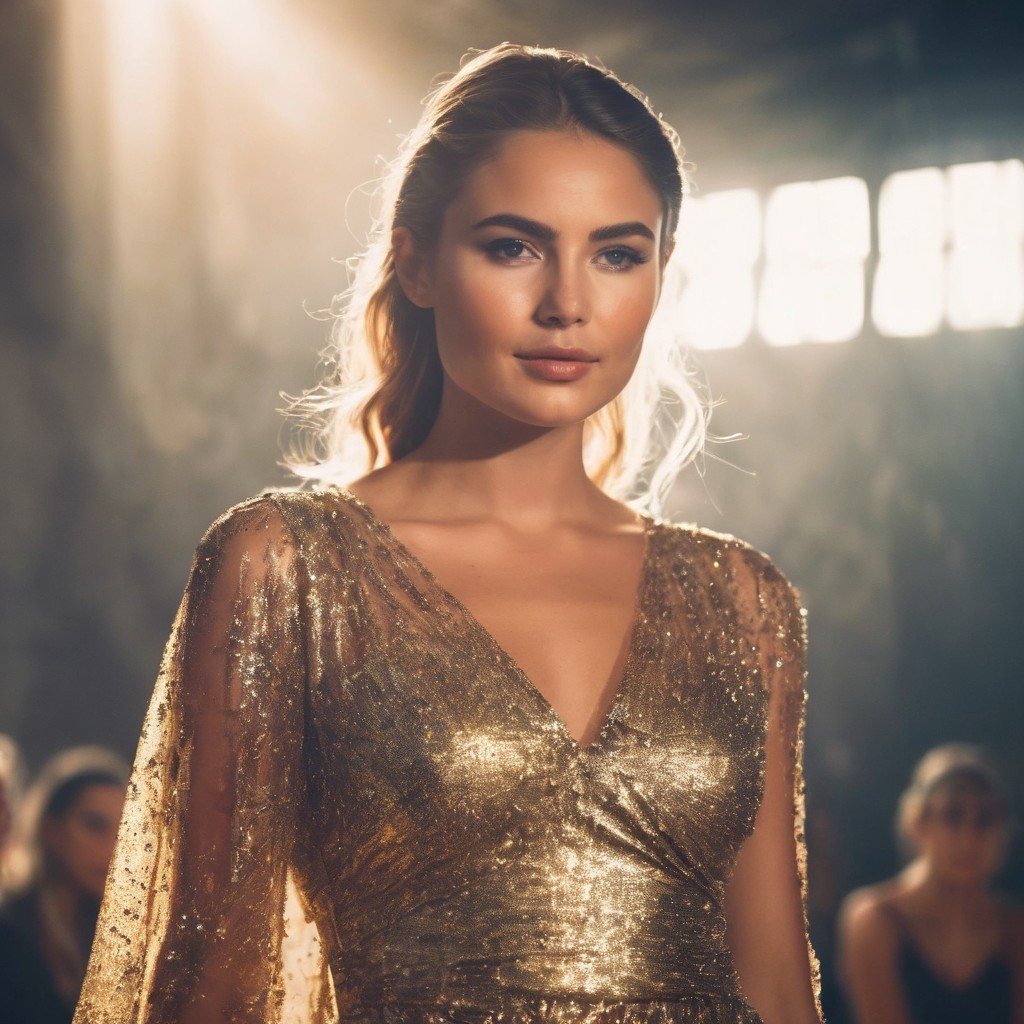}
\caption{Pool C $N{=}4$: VioletXL + catwalk + god rays + old gold. \emph{Left to right}: ambient \texttt{merge}, \texttt{composite}, uniform \LoRASpace{}.}
\label{fig:qual_c_n4}
\end{figure}

\begin{figure}[!htbp]
\centering
\includegraphics[width=0.31\linewidth]{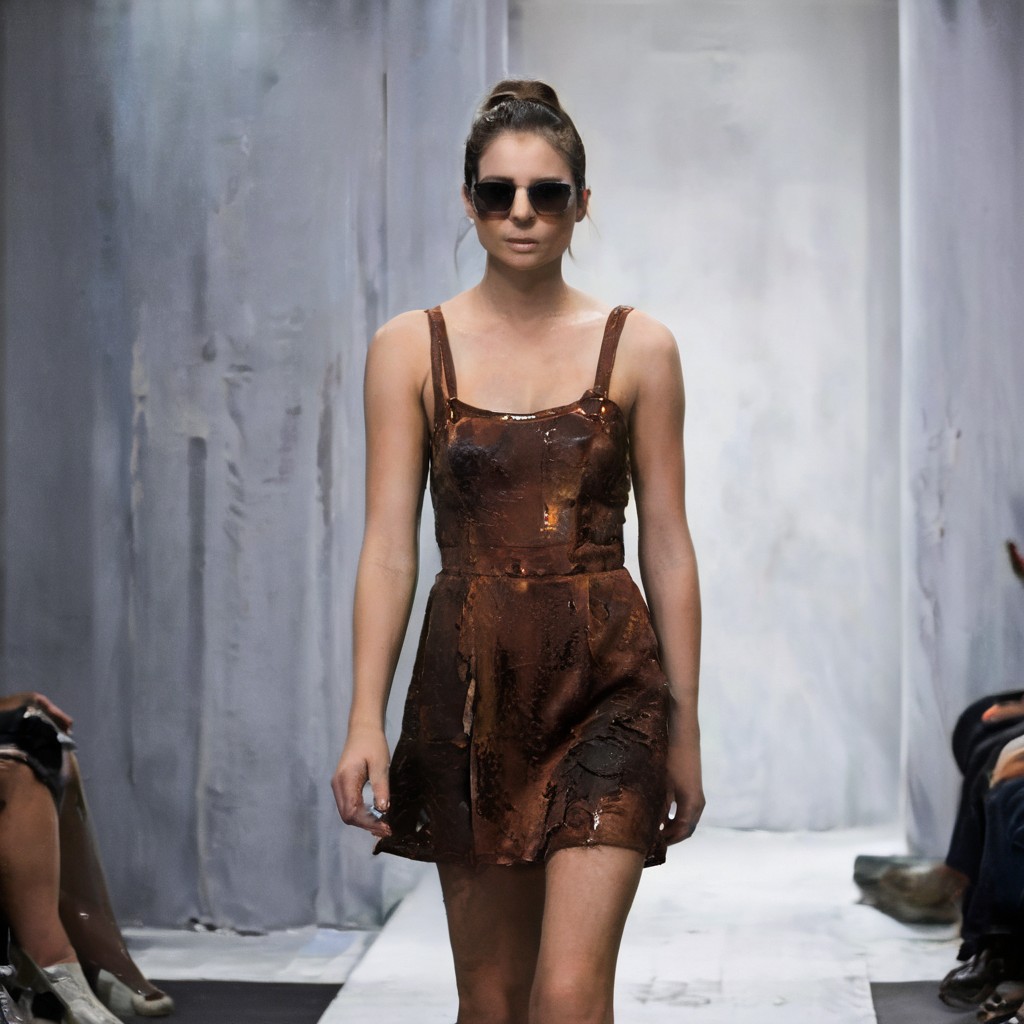}\hspace{6pt}%
\includegraphics[width=0.31\linewidth]{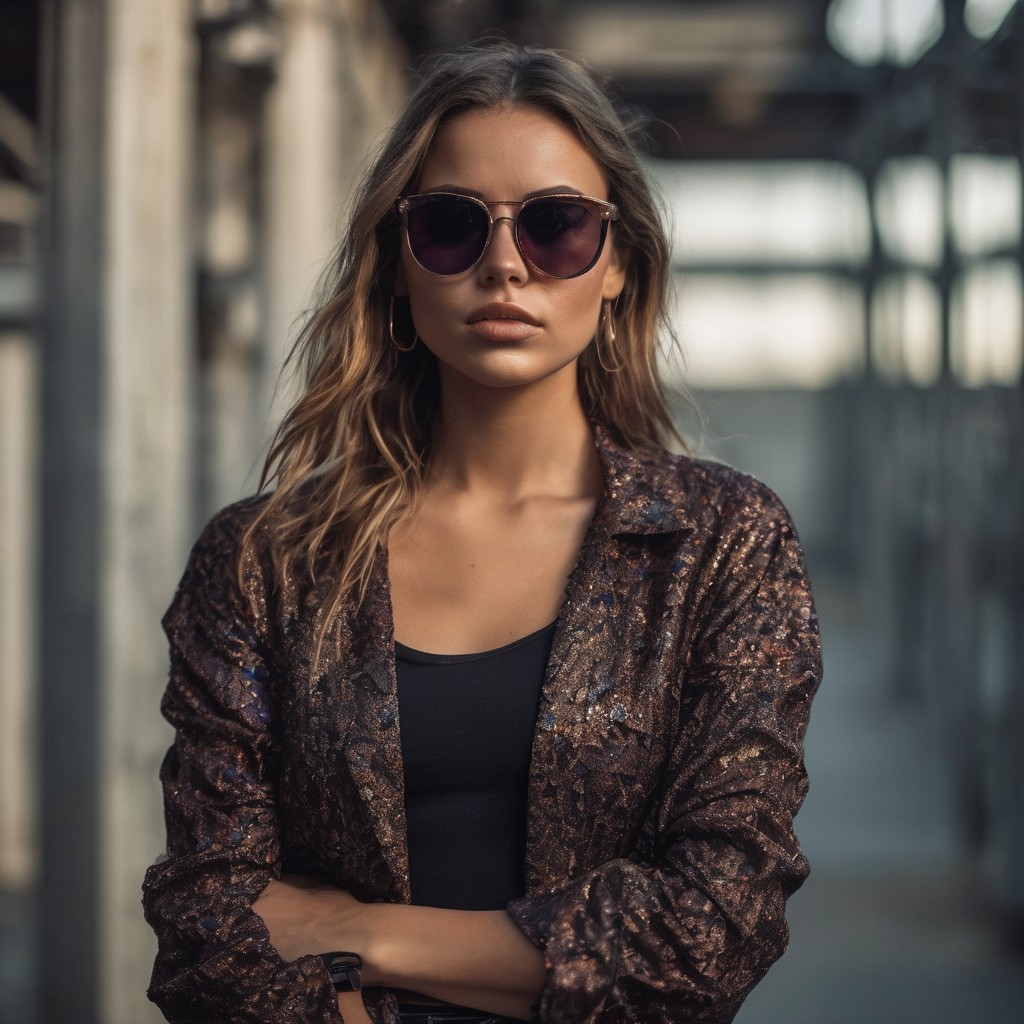}\hspace{6pt}%
\includegraphics[width=0.31\linewidth]{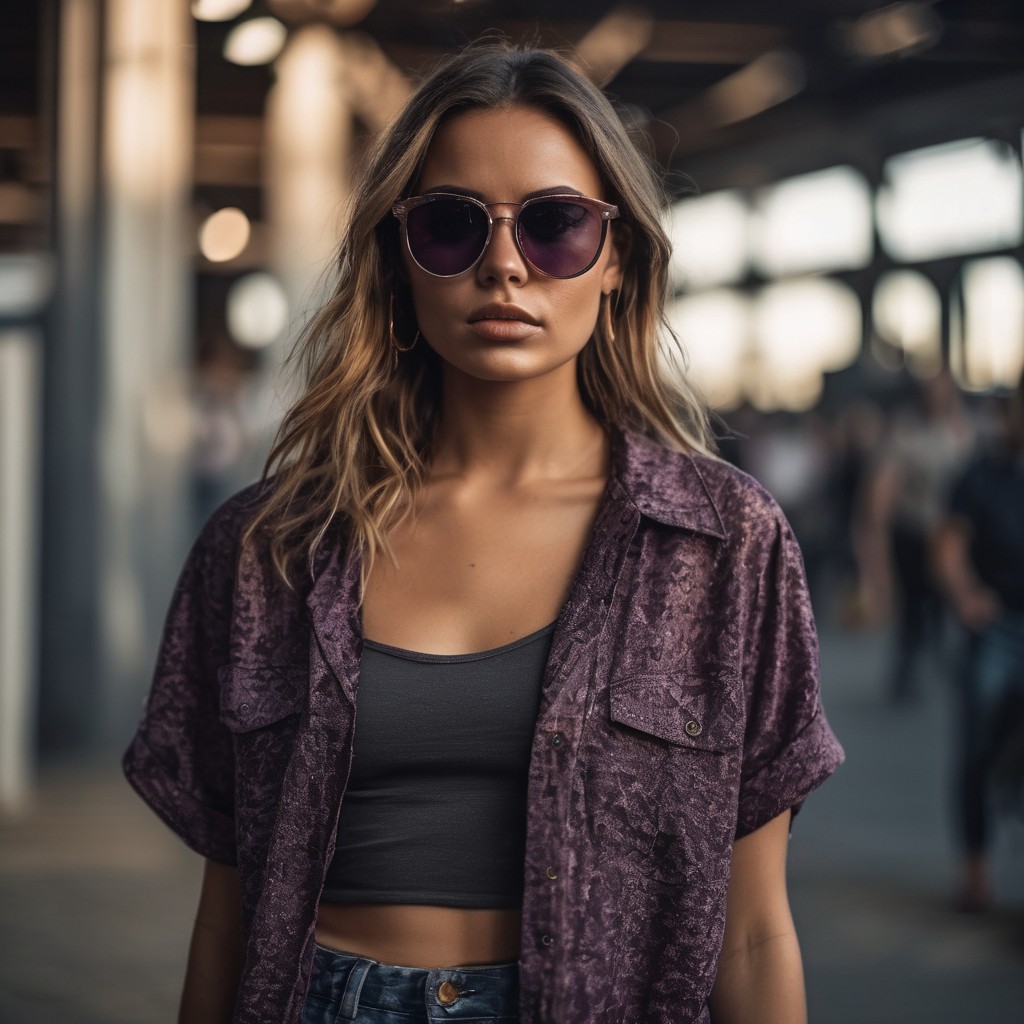}
\caption{Pool C $N{=}6$ (high-cardinality): VioletXL + catwalk + chiaroscuro + corrosion + fear + sunglasses.}
\label{fig:qual_c_n6}
\end{figure}

\clearpage

\section{Synthetic pool scaling stress test}
\label{app:synthetic_scaling}

To probe the asymptotic regime beyond real LoRA pools, we run $K \in \{4, 8, 16, 32, 64\}$ random rank-8 LoRAs on SD\,1.5 ($N{=}4$ active per call, $5$ DPM-Solver++~\citep{lu2025dpmsolverpp} steps, $512^2$).
\texttt{composite}'s per-step time grows from $0.119$ to $0.681$\,s over the $16\times$ pool growth at fixed $N$, while \LoRASpace{} stays within $0.132$ to $0.164$\,s and is $4.8\times$ faster at $K{=}64$.
Load-time GPU VRAM is constant at $2.16$ to $2.18$\,GB for \LoRASpace{} and grows linearly for \texttt{composite}.
The crossover from composite-faster to LoRASpace-faster occurs at $K \approx 8$; the gap widens monotonically thereafter.

\paragraph{Note on composite scaling at fixed $N$.}
The per-step asymptotic in Tab.~\ref{tab:complexity_per_method} is $O(NLd^2 S)$, independent of $K$. The $\sim K$ growth measured here at $N{=}4$ reflects implementation-level overhead in the diffusers PEFT-loading path, which iterates over loaded LoRAs per step at small step counts. At the $50$ denoising steps used in Pool C body measurements this overhead amortises and the per-step cost is dominated by the $O(NLd^2S)$ term; the body Pool C numbers (Tab.~\ref{tab:speed_memory_full}) are consistent with that asymptotic.

\section{Switch baseline scoreboard}
\label{app:switch_comparison}

We exclude single-active routing methods such as Switch~\citep{fleshman2025spectr} from the main composition comparison because they activate one LoRA at a time and never combine multiple LoRAs in a single forward pass. They are also order-sensitive by construction: rotating which LoRA fires at which step changes the output. For completeness, Table~\ref{tab:switch_scoreboard} reports Switch on the same pools as the main quality results.

\begin{table}[h]
\centering
\small
\caption{Switch baseline (one active LoRA at a time, same protocol as main results). Reported for completeness; not directly comparable to composition methods.}
\label{tab:switch_scoreboard}
\begin{tabular}{lcccc}
\toprule
Pool & Setting & CLIP-T $\uparrow$ & ImageReward $\uparrow$ & DINO (vs.\ comp.) \\
\midrule
Pool A SD\,1.5 & mean $N=2$ to $5$, 240 combos & 35.81 & 1.14 & 0.73 \\
Pool C SDXL    & mean $N=2$ to $6$, 200 combos & 32.10 & 1.15 & 0.79 \\
Pool B SD\,1.5 & $N=7$, 48 combos            & 36.43 & 1.23 & 0.69 \\
Pool B SD\,1.5 & $N=8$, 48 combos            & 36.10 & 1.10 & 0.65 \\
\midrule
LoRA-Retriever & Macro Task / samp.s$^{-1}$  & \multicolumn{3}{c}{$51.54$ Macro Task at $1.82$ samp.s$^{-1}$} \\
\bottomrule
\end{tabular}
\end{table}

The order-sensitivity manifests on Pool C at $N{=}7$ as a mean DINO self-similarity of $0.71$ between canonical and reversed LoRA orderings, with $10/10$ tested combos falling below the $0.95$ order-invariant threshold. Composition methods (merge, composite, \LoRASpace{}) are order-invariant by design.

\section{LLM D ablation: saturation and rank-orthogonality}
\label{app:llm_d_ablation}

The rank-overlap framework predicts two qualitatively different regimes for $D$ as $K_{\mathrm{pool}}$ grows. In the rank-bound regime ($\sum_i r_i < d$), $D$ grows sub-linearly with $K$ when LoRAs share structure. In the saturation regime ($\sum_i r_i \ge d$), $D$ is capped at the layer width and each new LoRA adds at most one new direction.

We measure both regimes empirically.

\paragraph{LoRA-Retriever near-orthogonality (rank-bound regime).}
The LLaMA-2-7B LoRA-Retriever pool ($K{=}48$)~\citep{zhao2024loraretriever} has uniform $r_i = 8$, $\sum_i r_i = 384$, against $d = 4096$. Across pool size sweeps with $K \in \{2, 4, 8, 16, 32, 48\}$, the rank-fill ratio stays at $\rho \ge 0.94$ throughout, with $\beta = 0.94$. LoRAs trained on disjoint NLP tasks (sentiment, machine translation, reading comprehension) contribute nearly orthogonal directions, consistent with tasks that share little structure (the pool sits at its spectrum-matched null, App.~\ref{app:rho_decomposition}).

\paragraph{LoRA-Hub saturation crossing (saturation regime).}
The LoRA-Hub FLAN-T5-Large pool~\citep{huang2024lorahub} has $K=196$ LoRAs with uniform $r_i = 16$ and layer width $d = 1024$. Saturation occurs at $K^{\star} = d / r = 64$, where $K \cdot r = d$. We measure $\bar{D}$ across $K \in \{4, 8, 16, 32, 64, 128, 196\}$. Below $K^{\star}$, each new LoRA adds approximately $11$ new dimensions to $\bar{D}$. Above $K^{\star}$, each new LoRA adds approximately $1$ new dimension, consistent with the saturation prediction. At the full pool $K=196$, $\bar{D} = 810$, which is $79\%$ of layer width: the subspace never fully fills even at $3{\times}$ saturation $K$, leaving residual capacity for additional LoRAs.

This is the first measurement we are aware of that crosses the saturation boundary in a real LoRA pool. It validates the framework's prediction of a regime change at $K = d/r$ and provides an empirical anchor for what to expect at very large pool sizes. Our diffusion pools are already width-saturated on most layers (App.~\ref{app:complexity}), but none was swept across the boundary, so LoRA-Hub is the only pool where we observe the crossing itself.

\section{LLM cross-modality: full scoreboard and extended discussion}
\label{app:llm_full}

\begin{table}[h]
\centering
\small
\caption{Full LoRA-Retriever benchmark~\citep{zhao2024loraretriever} on LLaMA-2-7B~\citep{touvron2023llama2} ($48$ LoRAs, $2{,}395$ samples, $14$ methods). \textbf{Bold} = best non-Oracle; \underline{underline} = best non-\LoRASpace{} non-Oracle.}
\label{tab:llm_full}
\begin{tabular}{lrrr}
\toprule
Method & Macro Task & Macro Domain & samp/s \\
\midrule
\texttt{base} (no LoRA)                & 4.24  & 4.17  & 0.60 \\
\texttt{oracle} (task-matched)            & 57.54 & 59.55 & 1.91 \\
\midrule
\texttt{merge} (Task Arithmetic)          & 45.72 & 46.96 & \textbf{2.18} \\
\texttt{ties}                             & 48.71 & 51.39 & 1.93 \\
\texttt{dare}                             & 40.86 & 41.22 & 2.04 \\
\texttt{knots}                            & 45.44 & 47.01 & 0.31 \\
\texttt{spectr}                           & 41.90 & 42.74 & 0.97 \\
\texttt{lorahub}                          & 46.19 & 46.68 & 0.12 \\
\texttt{switch}                           & 51.54 & 52.69 & 1.82 \\
\texttt{loraretriever-mixture} & \underline{53.55} & \underline{56.43} & 1.47 \\
\texttt{loraretriever-fusion}             & 43.56 & 44.14 & 1.77 \\
\midrule
\texttt{loraspace-fullpool} (no retrieval) & 50.49 & 53.00 & 1.48 \\
\texttt{loraspace-instructor}             & 54.01 & 56.79 & 1.44 \\
\texttt{loraspace-hybrid}                 & \textbf{54.51} & \textbf{57.20} & 1.52 \\
\bottomrule
\end{tabular}
\end{table}

\paragraph{Why \texttt{loraretriever-fusion} underperforms \texttt{loraretriever-mixture}.}
Fusion mode forms a fused left factor $\bar{A} = \sum_p w_p A_p$ and fused right factor $\bar{B} = \sum_p w_p B_p$, then applies $\bar{B}\bar{A}$: this introduces cross-products $B_p A_q$ for $p \ne q$, which are not part of the target delta sum $\sum_p w_p B_p A_p$.
These cross-terms are approximation errors that degrade composition quality relative to the clean mixture mode, explaining the substantial score difference ($43.56$ vs.\ $53.55$ Macro Task).

\paragraph{Ablation decomposition.}
The \texttt{loraspace-hybrid} score ($54.51$) decomposes as:
\texttt{loraspace-fullpool} ($50.49$) $+$ INSTRUCTOR retrieval gate ($+3.52$) $+$ calibrated spectral routing over retrieved subset ($+0.50$).
The substrate alone ($50.49$) is above all static merges but below the retrieval-augmented baselines. The retrieval gate is the dominant quality driver, as expected for a pool of $48$ highly diverse task LoRAs.

\paragraph{Per-domain breakdown.}
\texttt{loraspace-hybrid} beats Oracle on five of ten domains (commonsense, closed-book QA, sentiment, coreference, translation) where blending complementary LoRAs helps.
Losses concentrate on format-sensitive domains (reading comprehension, NLI, structured generation) where blending dilutes the precise output format required by the task-matched LoRA.

\section{Limitations}
\label{app:limitations}

This appendix collects the empirical scope and assumptions of the work referenced from \S\ref{sec:conclusion}.

\label{sec:conclusion:limitations}

\paragraph{Pool coverage and post-saturation behaviour.}
The empirical characterisation in \S\ref{sec:experiments:overlap} covers three diffusion pools (Pool A, B, C) and two LLM pools (LoRA-Retriever and LoRA-Hub). The diffusion pools already exceed the layer width on most layers (App.~\ref{app:complexity}), but only the LoRA-Hub pool is swept across the saturation boundary (Appendix~\ref{app:llm_d_ablation}), so how $D$ grows past it on diffusion pools is extrapolated rather than measured.

\paragraph{Changing pools.}
The substrate is computed for a given pool, and in deployment pools change. An unchanged substrate has no tolerance budget for a new LoRA. Projected onto the existing bases, a held-out Pool C adapter keeps only $0.30$ of its energy on average, and a larger existing pool does not fix this (App.~\ref{app:dynamic_pools}). Extending the bases with the new adapter's residual directions admits it in $33$\,s on CPU without recomputing the pool, and reproduces the full-recompute images, but the composed error grows slowly as arrivals accumulate ($3.3\%$ after the first to $6.4\%$ after the eleventh, against $0.84\%$ for a fresh build). The update rule we use is to extend on arrival, monitor composed error, and rebuild periodically off the serving path. Removing an adapter needs no recomputation, because the bases still span the remaining adapters.

\paragraph{Interpreting the rank-fill ratio.}
The rank-fill ratio may combine within-LoRA redundancy with cross-LoRA overlap. App.~\ref{app:rho_decomposition} quantifies the two with the decomposition $\rho = \rho_{\mathrm{intra}} \times \rho_{\mathrm{cross}}$ and a spectrum-matched null.

\paragraph{Efficiency baselines.}
Efficiency conclusions depend on the baseline implementation. Our merge, composite and single-pass ambient baselines use the PEFT runtime. Against a single-pass ambient mixture, the dynamic-mode gain at $\tau = 0.99$ comes from backbone-pass avoidance rather than from arithmetic inside the subspace (App.~\ref{app:pool_c_speed}).

\paragraph{Seed coverage of routing and mechanism ablations.}
The routing ablation (\S\ref{sec:experiments:routing}) is reported as mean $\pm$ across-seed standard deviation over three random seeds ($240$ stratified combinations per seed); the mechanism ablation (\S\ref{sec:experiments:quality}) is reported at a single seed. The main quality table (\S\ref{sec:experiments:quality}) reports CLIP-T and ImageReward across three random seeds with small across-seed standard deviations; GPT-4o evaluation is from one seed due to API cost. These are precision limitations rather than soundness ones: the across-seed standard deviations do not cross the inter-method gaps on ImageReward (where the routing pattern is established), and the pattern is consistent across $N$.

\paragraph{Layer scope.}
The operator is currently formulated for $1 \times 1$ linear projections, which covers all attention and most MLP modules in the diffusion and LLM backbones we test. Extensions to convolutional layers with kernel size larger than one, and to non-linear LoRA formats (DoRA, AdaLoRA), are out of scope here.

\paragraph{Router scope.}
The four routers we evaluate (uniform, spatial, temporal, semantic, plus the external retrieval coupling) are sufficient to demonstrate that the operator interface admits arbitrary $\{w_i\}$, but they are simple hand-designed examples and not exhaustive. Learned routers, in particular, may exploit the shared subspace structure in ways not captured by hand-designed signals. We see this as a natural extension rather than a limitation of the framework itself.

\paragraph{Asset and identity considerations.}
The three Pool A character LoRAs (also used in Pool B) are public CivitAI LoRAs that reproduce the likeness of real people, and CivitAI has since removed them. We use them only as inputs to the composition benchmark. We do not show their generations, name the people, or redistribute the weights, and our released pool manifests list them by SHA-256 hash only. We do not propose new likeness models. Downstream use should follow source platforms' policies and applicable law.

\section{LoRA pools and prompts}
\label{app:pools_and_prompts}

This appendix lists every LoRA used in our diffusion and LLM pools, with category groupings, trigger words, and example combo prompts. SHA-256 hashes and stable source URLs are in the per-pool JSON manifests in the \texttt{pools/} directory of the released code.

\subsection{Pool A (SD\,1.5, 11 LoRAs, reality style)}
\label{app:pool_a_details}

Five categories: 3 character + 2 clothing + 2 style + 2 background + 2 object. All LoRAs at $r{=}32$, $\alpha{=}16$ unless noted.

\begin{table}[h]
\centering
\small
\caption{Pool A LoRAs with trigger words. Source: CivitAI (URLs and SHA-256 hashes in the pool manifests).}
\label{tab:pool_a_adapters}
\begin{tabular}{l l >{\raggedright\arraybackslash}p{7cm}}
\toprule
ID & Name & Trigger words \\
\midrule
\multicolumn{3}{l}{\emph{character}} \\
\texttt{character\_1} & Real-person likeness LoRA & identity trigger and appearance descriptors (omitted) \\
\texttt{character\_2} & Real-person likeness LoRA & identity trigger and appearance descriptors (omitted) \\
\texttt{character\_3} & Real-person likeness LoRA & identity trigger and appearance descriptors (omitted) \\
\midrule
\multicolumn{3}{l}{\emph{clothing}} \\
\texttt{clothing\_1} & Thai University Uniform & mahalaiuniform, white shirt short sleeves, black pencil skirt \\
\texttt{clothing\_2} & School Dress & school uniform, white shirt, red tie, blue pleated microskirt \\
\midrule
\multicolumn{3}{l}{\emph{style}} \\
\texttt{style\_1} & Japanese Film Color Style & film overlay, film grain \\
\texttt{style\_2} & Bright Style & bright lighting \\
\midrule
\multicolumn{3}{l}{\emph{background}} \\
\texttt{background\_1} & Library Bookshelf & lib\_bg, library bookshelf \\
\texttt{background\_2} & Forest & slg, river, forest \\
\midrule
\multicolumn{3}{l}{\emph{object}} \\
\texttt{object\_1} & Umbrella & transparent umbrella \\
\texttt{object\_2} & Bubble Gum & bubble gum \\
\bottomrule
\end{tabular}
\end{table}

\paragraph{Pool A prompt template.} Identical across all combos in the pool; combo selection only changes which trigger lists get concatenated.
\begin{quote}\small\ttfamily
RAW photo, subject, 8k uhd, dslr, high quality, Fujifilm XT3, half-length portrait from knees up, $\langle$trigger\_1$\rangle$, $\langle$trigger\_2$\rangle$, \dots
\end{quote}
For the negative prompt we use the 44-token ComposLoRA reality negative (artifact + style-contamination + anatomy-failure tokens).

\paragraph{Example Pool A combos.}
\begin{itemize}[leftmargin=*,topsep=2pt,itemsep=1pt]
\item $N{=}2$, \texttt{[character\_1, background\_2]}: \emph{``RAW photo, subject, \dots, $\langle$character\_1 triggers$\rangle$, slg, river, forest''}.
\item $N{=}5$, \texttt{[character\_1, clothing\_1, style\_1, background\_1, object\_1]}: \emph{``RAW photo, subject, \dots, $\langle$character\_1 triggers$\rangle$, mahalaiuniform, white shirt short sleeves, black pencil skirt, film overlay, film grain, lib\_bg, library bookshelf, transparent umbrella''}.
\end{itemize}

\subsection{Pool B (SD\,1.5, 17 LoRAs, Pool A + 6 high-cardinality)}
\label{app:pool_b_details}

Pool A's 11 LoRAs plus 6 new ones across 3 new categories: detail (2), pose (2), expression (2). Probes the high-cardinality regime ($N \in \{7, 8\}$) where rank-fill is lowest ($\rho = 0.47$ at $K{=}17$).

\begin{table}[h]
\centering
\small
\caption{Pool B additional LoRAs (Pool A's 11 plus 6 new).}
\label{tab:pool_b_adapters}
\begin{tabular}{l l >{\raggedright\arraybackslash}p{7cm}}
\toprule
ID & Name & Trigger words \\
\midrule
\multicolumn{3}{l}{\emph{detail (quality enhancers, no explicit trigger)}} \\
\texttt{detail\_1} & Detail Tweaker & (none) \\
\texttt{detail\_2} & Add More Details & (none) \\
\midrule
\multicolumn{3}{l}{\emph{pose}} \\
\texttt{pose\_1} & Shirt Tug Pose & shirt tug \\
\texttt{pose\_2} & Covering Eyes Pose & covering eyes \\
\midrule
\multicolumn{3}{l}{\emph{expression}} \\
\texttt{expression\_1} & Screaming / Open Mouth & open\_mouth, screaming \\
\texttt{expression\_2} & Sharp Eyes / Kubrick Stare & sharp\_eyes \\
\bottomrule
\end{tabular}
\end{table}

\paragraph{Pool B trigger deduplication.} At $N \ge 7$, the body of the prompt approaches CLIP-L's 77-token limit. The Pool B body experiments use a \texttt{compact} trigger-deduplication rule (App.~\ref{app:truncation}) that drops repeated tokens across active LoRAs before assembly.

\subsection{Pool C (SDXL, 17 LoRAs)}
\label{app:pool_c_details}

Six categories: 3 character ($r{=}32$, $\alpha{=}16$) + 3 background + 4 lighting + 4 material + 2 expression + 1 accessory (last five at $r{=}128$, $\alpha{=}1$). SDXL backbone, $1024^2$ generation.

\begin{table}[h]
\centering
\small
\caption{Pool C LoRAs (SDXL).}
\label{tab:pool_c_adapters}
\begin{tabular}{l l >{\raggedright\arraybackslash}p{7cm}}
\toprule
ID & Name & Trigger words \\
\midrule
\multicolumn{3}{l}{\emph{character}} \\
\texttt{character\_1} & SL Nobodies CamelliaXL & CamelliaXL, a young woman, detailed face \\
\texttt{character\_2} & SL Nobodies HelenaXL & HelenaXL, a young woman, detailed face \\
\texttt{character\_3} & SL Nobodies VioletXL & VioletXL, a young woman, detailed face \\
\midrule
\multicolumn{3}{l}{\emph{background}} \\
\texttt{background\_1} & Polyhedron Catwalk & catwalk, runway, fashion show \\
\texttt{background\_3} & Polyhedron Power Grid & energy, high voltage, power pole, transformer station, tangled cables \\
\texttt{background\_4} & Polyhedron Dandelions & dandelion, dandelion seed, yellow dandelion \\
\midrule
\multicolumn{3}{l}{\emph{lighting}} \\
\texttt{lighting\_2} & Polyhedron God Rays & god rays, sunbeams, sunbeam \\
\texttt{lighting\_3} & Polyhedron Chiaroscuro & chiaroscuro style, dim light \\
\texttt{lighting\_4} & Polyhedron Fog & fog, haze, foggy, mist \\
\texttt{lighting\_5} & Polyhedron Golden Hour & golden hour, sunset, sunrise \\
\midrule
\multicolumn{3}{l}{\emph{material}} \\
\texttt{material\_1} & Polyhedron Corrosion & corrosion, rust \\
\texttt{material\_2} & Polyhedron Old Wood & old wood, wooden planks, wooden wall, wooden floor \\
\texttt{material\_3} & Polyhedron Old Gold & gold, old gold \\
\texttt{material\_4} & Polyhedron Old Copper & old copper, copper, patina, copper pot \\
\midrule
\multicolumn{3}{l}{\emph{expression}} \\
\texttt{expression\_1} & Polyhedron Anger & anger, screaming, angry \\
\texttt{expression\_3} & Polyhedron Fear & fear, fearful, anxious \\
\midrule
\multicolumn{3}{l}{\emph{accessory}} \\
\texttt{accessory\_3} & Sunglasses on Top of Head & with sunglasses on top of head \\
\bottomrule
\end{tabular}
\end{table}

\paragraph{Example Pool C combos.}
\begin{itemize}[leftmargin=*,topsep=2pt,itemsep=1pt]
\item $N{=}4$, \texttt{[character\_1, background\_1, lighting\_2, material\_1]}: \emph{``RAW photo, subject, \dots, CamelliaXL, a young woman, detailed face, catwalk, runway, fashion show, god rays, sunbeams, sunbeam, corrosion, rust''}.
\item $N{=}6$, \texttt{[character\_2, background\_3, lighting\_4, material\_2, expression\_1, accessory\_3]}: \emph{``RAW photo, subject, \dots, HelenaXL, \dots, energy, high voltage, power pole, transformer station, tangled cables, fog, haze, foggy, mist, old wood, wooden planks, wooden floor, anger, screaming, angry, with sunglasses on top of head''}.
\end{itemize}

\subsection{LLM pool (LLaMA-2-7B, 48 task LoRAs, LoRA-Retriever benchmark)}
\label{app:llm_pool_details}

48 task-specialised LoRAs from the LoRA-Retriever benchmark~\citep{zhao2024loraretriever}, all $r{=}8$. We adopt the authors' 10-domain grouping verbatim.

\begin{table}[h]
\centering
\small
\caption{LoRA-Retriever 48 task LoRAs by domain.}
\label{tab:llm_pool_adapters}
\begin{tabular}{l >{\raggedright\arraybackslash}p{9cm}}
\toprule
Domain (count) & Tasks \\
\midrule
closed-book QA (4)        & \texttt{arc\_challenge}, \texttt{arc\_easy}, \texttt{natural\_questions}, \texttt{trivia\_qa} \\
commonsense (4)           & \texttt{copa}, \texttt{hellaswag}, \texttt{piqa}, \texttt{story\_cloze} \\
coreference (2)           & \texttt{definite\_pronoun\_resolution}, \texttt{wsc} \\
NLI (10)                  & \texttt{anli\_r1}, \texttt{anli\_r2}, \texttt{anli\_r3}, \texttt{cb}, \texttt{mnli\_matched}, \texttt{mnli\_mismatched}, \texttt{qnli}, \texttt{rte}, \texttt{snli}, \texttt{wnli} \\
paraphrase (4)            & \texttt{glue\_mrpc}, \texttt{glue\_qqp}, \texttt{paws\_wiki}, \texttt{stsb} \\
RC w/ commonsense (2)     & \texttt{cosmos\_qa}, \texttt{record} \\
reading-comprehension (6) & \texttt{bool\_q}, \texttt{drop}, \texttt{multirc}, \texttt{openbookqa}, \texttt{squad\_v1}, \texttt{squad\_v2} \\
sentiment (4)             & \texttt{imdb\_reviews}, \texttt{sentiment140}, \texttt{sst2}, \texttt{yelp\_polarity\_reviews} \\
structured-to-text (4)    & \texttt{common\_gen}, \texttt{dart}, \texttt{e2e\_nlg}, \texttt{web\_nlg\_en} \\
translation (8)           & \texttt{para\_crawl\_enes}, \texttt{wmt14\_enfr}, \texttt{wmt16\_translate\_csen}, \texttt{wmt16\_translate\_deen}, \texttt{wmt16\_translate\_fien}, \texttt{wmt16\_translate\_roen}, \texttt{wmt16\_translate\_ruen}, \texttt{wmt16\_translate\_tren} \\
\bottomrule
\end{tabular}
\end{table}

LLM evaluation uses task-native prompts directly from the LoRA-Retriever benchmark (no LoRA trigger tokens; each LoRA is task-specialised, not concept-keyed). Multi-task composition replaces ``compose multiple LoRAs for one image'' with ``compose multiple LoRAs for one prompt'': at inference, we activate either the full 48-pool substrate (\texttt{loraspace-fullpool}) or a top-$K$ retrieved subset (\texttt{loraspace-instructor} / \texttt{loraspace-hybrid}).

\subsection{Combo selection protocol}
\label{app:combo_selection}

For each pool we use stratified random sampling over $\binom{K}{N}$ combos with category-balance constraints (each combo activates at most one LoRA per category where pool design allows). Pool A and Pool B use 240 combos per seed across $N \in \{2, 3, 4, 5\}$. Pool C uses 200 combos per method (40 per $N$ across $N \in \{2, \dots, 6\}$). Identical combo lists are shared across all methods within a pool to enable per-combo paired comparison.

\end{document}